\documentclass[11pt]{article}
\usepackage{algorithm}
\usepackage{algorithmic}

\usepackage{graphicx, url}
\usepackage[numbers]{natbib}
\usepackage{tikz}
\usepackage{hyperref}  
\hypersetup{colorlinks=true,citecolor=blue,linkcolor=red} 

\usepackage{amsmath, amssymb}
\usepackage{amsthm}

\usepackage{thmtools}
\usepackage{thm-restate}

\newtheorem{theorem}{Theorem}[section]
\newtheorem{lemma}[theorem]{Lemma}

\newtheorem{fact}[theorem]{Fact}
\newtheorem{claim}[theorem]{Claim}

\newtheorem{assumption}[theorem]{Assumption}

\newtheorem{definition}[theorem]{Definition}

\usepackage[lmargin=1in,rmargin=1in,tmargin=0.8in,bmargin=0.8in]{geometry}
\graphicspath{{./figs/}}

\definecolor{b2}{RGB}{51,153,255}
\definecolor{mygreen}{RGB}{80,180,0}

\newcommand{\Daogao}[1]{\textcolor{mygreen}{[Daogao: #1]}}

\newcommand{\wh}{\widehat}

\renewcommand{\epsilon}{\varepsilon}
\renewcommand{\phi}{\varphi}

\newcommand{\R}{\mathbb{R}}
\newcommand{\HF}{\hat{F}}
\newcommand{\hT}{\hat{T}}

\newcommand{\calK}{\mathcal{K}}

\newcommand{\calP}{\mathcal{P}}

\renewcommand{\hat}{\wh}

\renewcommand{\d}{\mathrm{d}}

\DeclareMathOperator*{\E}{\mathbb{E}}

\DeclareMathAlphabet{\mathpzc}{OT1}{pzc}{m}{it}

\newcommand{\smin}{\textup{smin}}

\definecolor{burntorange}{rgb}{0.8, 0.33, 0.0}

\newcommand{\oracle}{\mathcal{O}}
\newcommand{\alg}{\mathcal{A}}

\newcommand{\calA}{\mathcal{A}}
\newcommand{\calD}{\mathcal{D}}
\newcommand{\calF}{\mathcal{F}}
\newcommand{\FP}{F_{\calP}}
\newcommand{\FD}{F_{\calD}}
\newcommand{\CLSI}{C_{\mathrm{LSI}}}

\newcommand{\drift}{\mathrm{drift}}

\newcommand{\frozen}{\mathrm{frozen}}
\newcommand{\calN}{\mathcal{N}}

\newcommand{\Lap}{\mathrm{Lap}}
\newcommand{\nSG}{\mathrm{nSG}}
\newcommand{\SG}{\mathrm{SG}}

\title{Private (Stochastic) Non-Convex Optimization Revisited:\\ Second-Order Stationary Points and Excess Risks}

\author{%
Arun Ganesh \thanks{ Google Research, \tt{arunganesh@google.com}}
\and
 Daogao Liu \thanks{ University of Washington, most of this work was done while interning at Google, \tt{dgliu@uw.edu}}
 \and
 Sewoong Oh \thanks{University of Washington and Google Research, \tt{sewoong@cs.washington.edu}}
 \and
 Abhradeep Thakurta \thanks{Google Research, \tt{athakurta@google.com}}
}
\begin{document}

\maketitle

\begin{abstract}%
We consider the problem of minimizing a non-convex objective while preserving the privacy of the examples in the training data.
Building upon the previous variance-reduced algorithm SpiderBoost, we introduce a new framework that utilizes two different kinds of gradient oracles. The first kind of oracles can estimate the gradient of one point, and the second kind of oracles, less precise and more cost-effective, can estimate the gradient difference between two points.
SpiderBoost uses the first kind periodically, once every few steps, while our framework proposes using the first oracle whenever the total drift has become large and relies on the second oracle otherwise.
This new framework ensures the gradient estimations remain accurate all the time, resulting in improved rates for finding second-order stationary points.

Moreover, we address a more challenging task of finding the global minima of a non-convex objective using the exponential mechanism. Our findings indicate that the regularized exponential mechanism can closely match previous empirical and population risk bounds, without requiring smoothness assumptions for algorithms with polynomial running time.
Furthermore, by disregarding running time considerations, we show that the exponential mechanism can achieve a good population risk bound and provide a nearly matching lower bound.
\end{abstract}

\section{Introduction}
\label{sec:intro} 
Differential privacy \citep{DMNS06} is a standard privacy guarantee for training machine learning models. Given a randomized algorithm $\calA: P^* \rightarrow R$, where $P$ is a data domain and $R$ is a range of outputs, we say $\calA$ is $(\epsilon, \delta)$-differentially private (DP) for some $\varepsilon\geq0$ and $\delta\in[0,1]$ if for any neighboring datasets $\calD, \calD' \in P^*$ that differ in at most one element and any $\mathcal{R} \subseteq R$, the distribution of the outcome of the algorithm, e.g., pair of models trained on the respective datasets, are similar:
\[\Pr_{x \sim \calA(\calD)}\left[x \in \mathcal{R}\right] \leq e^\epsilon \Pr_{x \sim \calA(\calD')}\left[x \in \mathcal{R}\right] + \delta.\]
Smaller $\varepsilon$ and $\delta$ imply the distributions are closer; hence, an adversary accessing the trained model cannot tell with high confidence whether an example $x$ was in the training dateset. 
Given this measure of privacy, we consider the problem of optimizing a non-convex loss while ensuring a desired level of privacy. In particular, suppose we are given a dataset $\calD=\{z_1,\ldots,z_n\}$ drawn i.i.d. from underlying distribution $\calP$.
Each loss function $f(\cdot;z):\calK\to \R$ is $G$-Lipschitz over the convex set $\calK\subset\R^d$ of diameter $D$.
Let the population risk function be $\FP(x):=\E_{z\sim\calP}[f(x;z)]$ and the empirical risk function be  $\FD(x):=\frac{1}{n}\sum_{z\in \calD}f(x;z)$.
We also denote $F_{S}(x):=\frac{1}{|S|}\sum_{z\in S}f(x;z)$  for $S \subseteq \calD$.

Our focus is in minimizing non-convex risk functions, both empirical and population, 
which may have multiple local minima. 
Since finding the global optimum of a non-convex function can be challenging, an alternative goal in the field is to find stationary points: A first-order stationary point is a point with a small gradient of the function, and a second-order stationary point is a first-order stationary point where additionally the function has a positive or nearly positive semi-definite Hessian.
% \Arun{Reworded to clarify that a second-order stationary point is a stronger condition}
%More formally, progress towards privately finding a first-order stationary point is measured in($i$) the norm of the empirical gradient at the solution $x$, i.e., $\| \nabla \FD(x)\|$, and ($ii$) the norm of the population gradient, i.e., $\|\nabla \FP(x)\|$. 
%We say a point $x$ is a Second-Order Stationary Point (SOSP), or a local minimum of a twice differentiable function $g$ if $\|\nabla g(x)\|_2=0$ and $\smin(\nabla^2 g(x))\ge 0$.
%Exact second-order stationary points can be extremely challenging to find \cite{GHJY15}. Instead, it is common to measure the progress in terms of how well the solution approximates an SOSP.
%\begin{definition}[approximate-SOSP, \cite{AZB+17}]
%We say $x\in\R^d$ is an $\alpha$-second order stationary point ($\alpha$-SOSP) for $\rho$-Hessian Lipschitz function $g$, if
%\begin{align*}
%    \|\nabla g(x)\|_2\le\alpha \;\bigwedge\; \smin(\nabla^2 g(x))\ge -\sqrt{\rho \alpha}.
%\end{align*}
%\end{definition}
As first order stationary points can be saddle points or even a local maximum, we focus on the problem of finding a second order stationary point, i.e., a local minimum, privately. 
Existing works in finding approximate SOSP privately only give guarantees for the empirical function $\FD$. We improve upon the state-of-the-art result for empirical risk minimization and give the first guarantee for the population function $\FP$. This requires standard assumptions on bounded Lipschitzness, smoothness, and Hessian Lipschitzness, which we make precise in Section~\ref{sec:preliminary} and in Assumption~\ref{assump:SOSP}.

Compared to finding a local minimum, finding a global minimum can be extremely challenging.  Progress towards finding the global minima is measured in the excess empirical risk,
$
    \E[\FD(x^{priv})]-\min_{x\in\calK}\FD(x),
$
and the excess population risk, 
$
    \E[\FP(x^{priv})]-\min_{x\in\calK}\FP(x)
$ for a private solution $x^{priv}$.  
We provide two approaches, in polynomial time and exponential time, that improve upon the state-of-the-art guarantees as measured in the excess risks for the respective families of computational complexity.

%As the functions are non-convex, it is challenging to get a small empirical and population risk in polynomial time.
%With a mild smoothness assumption, \cite{WCX19} can get a private algorithm with polynomial running time and with upper $O(\frac{d\sqrt{\log(1/\delta)}}{\epsilon^2\log n})$ for both empirical and population risk.
%If we allow an exponential running time, \cite{GTU22} demonstrated $\Tilde{O}(\frac{d}{n\epsilon})$ upper bound for the empirical loss along with a nearly matching lower bound. 
%They leave it an open problem to obtain a tight bound for population risk of non-convex functions.
% define various notions first and second 
% why does one go after FOSP and SOSP in non-convex? 
% what can we do in function value?

\begin{center}
    \begin{table}[t]
    \centering
    \begin{tabular}{|c||c|c|c|c|}
    \hline
    & %\multicolumn{2}{c|}{$\alpha$-FOSP}&
    \multicolumn{2}{c|}{$\alpha$-SOSP}&
    \multicolumn{2}{c|}{Excess population risk} \\
         &% emp & pop & 
         empirical & population &   poly-time & exp-time \\
         \hline 
        SOTA & %$\frac{d^\frac13}{n^\frac23}$ & $\frac{1}{n^\frac13 } + \frac{d^\frac14}{n^\frac12} $ &
        $\min(\frac{d^\frac14}{n^\frac12},\frac{d^\frac47}{n^\frac47})$ & N/A & $\frac{d}{\epsilon^2 \log n}^\spadesuit$  & N/A \rule[-2ex]{0pt}{6ex}\\
       Ours  & %$\frac{d^\frac13}{n^\frac23}$ & $\frac{1}{n^{\frac{1}{3}}}+\frac{d^{\frac{1}{3}}}{n^\frac23}$ &
       $\frac{d^\frac13}{n^\frac23}$ & $\frac{1}{n^{
      \frac13 }}+\Big(\frac{\sqrt{d}}{n}\Big)^{\frac37} $ & $\frac{d\log\log n}{\epsilon\log(n)}$ & $\frac{d}{n\epsilon}+
        \sqrt{\frac{d}{n}} $ \rule[1ex]{0pt}{3ex}\\
        LB & %$\frac{\sqrt{d}}{n}$ & $\frac{1}{\sqrt{n}} + \frac{\sqrt{d}}{n}$ &
        $\frac{\sqrt{d}}{n}$& $\frac{1}{\sqrt{n}} + \frac{\sqrt{d}}{n}$ & $\frac{d}{n\epsilon}+
       \sqrt{\frac{d}{n}}$ & $\frac{d}{n\epsilon}+
       \sqrt{\frac{d}{n}}$ \rule[-2ex]{0pt}{6ex} \\
        \hline
    \end{tabular}
    \caption{SOTA refers to the best previously known bounds  
    on $\alpha$ for $\alpha$-SOSP by \cite{WCX19}  and  on the excess population risk by \cite{WCX19}.  
    We introduce algorithm~\ref{alg:DP-Spider-META} that finds an $\alpha$-SOSP (columns $2$--$3$) with an improved rate.
    We show exponential mechanism can minimize the excess risk in polynomial time and exponential time, respectively (columns $4$ and $5$). 
    $^\spadesuit$ requires extra assumption on bounded smoothness. 
    The lower bounds for SOSP are  from \citep{ABG+22}, 
    and the lower bound on excess population risk is from Theorem~\ref{thm:risk_lb}. We omit logarithmic factors in $n$ and $d$. }
    \label{tab:1}
\end{table} 
\end{center} 

\subsection{Main results}
\label{sec:intro_result}

Our main contribution is a private non-convex optimization algorithm based on the variance-reduced SpiderBoost \citep{WJZ+19};   Algorithm~\ref{alg:DP-Spider-META} achieves improved rates on the approximation error for finding  SOSP of the empirical and population risks privately.  Table~\ref{tab:1} summarizes our main results.  
 
% \begin{table}[h]
%     \centering
%     \begin{tabular}{|c|c|c|c|}
%     \hline
%       \cite{WCX19}   &  $\frac{\exp(O(\beta))}{n^2\epsilon^2}+\frac{\exp(O(\beta))}{n}+\frac{d}{\beta} $ & GLM: $\frac{d^{1/4}}{\sqrt{n\epsilon}}$ & Thm 6 ,$\ell_1$, $1/\sqrt{n\epsilon}$ 
%       \\ \hline
%     \cite{WX19} & population gradient bound: $\frac{\sqrt{d}}{\sqrt{n}}$& &\\
%       \hline
%       \cite{ZCH+20}& population gradient bound:$\frac{d^{1/4}}{\sqrt{n}}$ & &  
%       \\ \hline
%       \cite{BGM21}   & Consider stationarity gap, worse than \cite{ZCH+20}  &  &\\
%       \hline
%       \cite{WJEG19} & empirical gradient bound: $\frac{d^{1/4}}{\sqrt{n}}$ (More efficient,  Table 1)& & \\
%     \hline
%     \cite{ZMLX20} & population loss but under stronger assumption & &  PL,weakly quasi-convex\\
%     \hline
%     \cite{ABG+22,TC22} & empirical gradient bound: $\Big(\frac{\sqrt{d}}{n\varepsilon}\Big)^{2/3}$ &  &
%     \\
%     \hline
%     \cite{SSTT21} & emprical gradient bound: $\frac{\mathrm{rank}^{1/4}}{\sqrt{n}\epsilon}$ & & \\
%     \hline
%     \end{tabular}
%     \caption{DP-SO (exclude ERM). Omit log dependence}
%     \label{tab:my_label}
% \end{table}

% We can also try to consider stationary gap.
% There are sampling results for stationary gap as well.
% Something wired: The intro of \cite{ZCH+20} claims \cite{WJEG19} only considers the empricial gradient bound.
% Previous results in DP:
% (i) SOSP:\cite{WCX19,WX20},  \cite{WCX19} seems to be wrong on the dependence on $p$. SOTA is $(d/n)^{4/7}$, and we want to get $O(d/n)^{2/3}$.\\
% A very recent ICLR submission: \cite{GW23}

\paragraph{Finding second-order stationary points.} Advances in private non-convex optimization have focused on finding a first-order stationary point (FOSP), whose performance is measured in 
($i$) the norm of the empirical gradient at the solution $x$, i.e., $\| \nabla \FD(x)\|$, and ($ii$) the norm of the population gradient, i.e., $\|\nabla \FP(x)\|$. We survey the recent progress in 
Appendix~\ref{app:related} in detail.

\begin{definition}[First-order stationary point] We say 
$x\in\R^d$ is a First-Order Stationary Point
(FOSP) of $g:\R^d\to\R$ iff $\nabla g (x) = 0 $. $x$ is an $\alpha$-FOSP of g, if $\|\nabla g (x)\|_2 \leq \alpha$.
\label{def:fosp}
\end{definition} 
Since FOSP can be a saddle point or a local maxima, finding a second-order stationary point  is desired. Exact second-order stationary points can be extremely challenging to find \citep{GHJY15}. Instead, progress is commonly measured in terms of how well the solution approximates an SOSP.

\begin{definition}[Second-order stationary point, \cite{AZB+17}] 
We say a point $x\in \R^d$ is a Second-Order Stationary Point (SOSP) of a twice differentiable function $g:\R^d \to{\mathbb R}$ iff $\|\nabla g(x)\|_2=0$ and $\nabla^2 g(x) \succeq 0$.
We say $x\in\R^d$ is an $\alpha$-SOSP for $\rho$-Hessian Lipschitz function $g$, if $
    \|\nabla g(x)\|_2\le\alpha \;\bigwedge\; \nabla^2 g(x) \succeq  -\sqrt{\rho \alpha}I \; . 
$
\label{def:sosp}
\end{definition} 
On the empirical risk $F_\calD$, the SOTA on privately finding $\alpha$-SOSP is by \cite{WCX19,WX20}, which achieves $\alpha=\tilde O(\min\{(\sqrt{d}/n)^{1/2},(d/n)^{4/7}\})$. 
In Theorem~\ref{thm:sosp_emp}, we show that the proposed Algorithm~\ref{alg:DP-Spider-META} achieves a rate bounded by $\alpha=\tilde O((\sqrt{d}/n)^{2/3})$, which improves over the SOTA in all regime.\footnote{We want $\alpha=o(1)$ and hence can assume $d\le n^2$.}
There remains a factor $(\sqrt{d}/n)^{-1/6}$ gap to a known lower bound of $\alpha=\Omega(\sqrt{d}/n)$ that holds even if privacy is not required  and even if finding only an $\alpha$-FOSP \citep{ABG+22}. 
On the population risk $F_\calP$, Algorithm~\ref{alg:DP-Spider-META} is the first private algorithm to guarantee finding an $\alpha$-SOSP with $\alpha=\tilde O(n^{-1/3} + (\sqrt{d}/n)^{3/7})$ in Theorem~\ref{thm:sosp_pop}. There is a gap to a known lower bound of $\alpha=\Omega(1/\sqrt{n}+\sqrt{d}/n)$ that holds even if privacy is not required  and even if finding only an $\alpha$-FOSP \citep{ABG+22}. 
%Since $\alpha$-SOSP implies $\alpha$-FOSP, our results in the first order follow immediately.
%On empirical risk, this matches the SOTA rate on $\alpha$-FOSP. 
%On population risk, this improves upon the best known rate of $\alpha=\tilde O(n^{-1/3} + (\sqrt{d}/n)^{1/2})$ on $\alpha$-FOSP by \citep{ABG+22}.  

\paragraph{Minimizing excess risk.} We also provide sampling-based algorithms that aims to tackle the ultimate objective of finding a private solution $x^{priv}\in\R^d$ that minimizes the excess \textsc{empirical risk}:
$
    \E[\FD(x^{priv})]-\min_{x\in\calK}\FD(x),
$
and the excess \textsc{population risk}, $\E[\FP(x^{priv})]-\min_{x\in\calK}\FP(x)$, where the expectation is over the randomness on the solution $x^{priv}$. 
%As the functions are non-convex, it is challenging to get a small empirical and population loss in polynomial time.
With a mild smoothness assumption, \cite{WCX19} achieves in polynomial time a bound of  $O({d\sqrt{\log(1/\delta)}}/(\epsilon^2\log n) )$ for both excess empirical and population risks.
In Table~\ref{tab:1} we omit excess empirical risk, as the bounds are the same. 
We introduce a sampling-based algorithm from the exponential mechanism, which runs in polynomial time and achieves excess empirical and population risks bounded by $O({d\sqrt{\log(1/\delta)}}/(\epsilon \log(nd) ) )$ with improved dependence on $\varepsilon$ (Theorem~\ref{thm:zeroth_polytime_empirical}). 
Moreover, we do not need the smoothness assumption required by \cite{WCX19}.

If we allow an exponential running time, \cite{GTU22} demonstrated $\Tilde{O}(d/(\epsilon n) )$ upper bound for non-convex excess empirical risks along with a nearly matching lower bound. 
It remained an open question to obtain a tight bound for the excess population risk. We close this gap by providing a nearly matching upper and lower bounds of  $\tilde \Theta (d/(\epsilon n) + \sqrt{d/n})$ for the excess population risk (Theorem~\ref{thm:risk_exp}). 

\subsection{Our techniques}
\label{sec:intro_technique}

\paragraph{Stationary points.} We propose a simple framework based on  SpiderBoost \citep{WJZ+19} and its private version \citep{ABG+22} that achieves the current best rate for finding the first order stationary point privately.
In SGD and its variants, we usually get an estimation $\Delta_t$ of the gradient $\nabla f(x_t)$. 
In the stochastic variance-reduced algorithm SpiderBoost, it only queries the gradient $\oracle_1(x_t)\approx\nabla f(x_t)$ directly every $q$ steps with some oracle $\oracle_1$, and for the other $q-1$ steps in each period, it queries the difference between two steps, that is $\oracle_2(x_t,x_{t-1})\approx\nabla f(x_t)-\nabla f(x_{t-1})$, and maintain $\Delta_t=\Delta_{t-1}+\oracle_2(x_t,x_{t-1})$.
One interpretation of the difference between these two kinds of oracles is that,
in many situations, one can treat $\oracle_1$ as more accurate and more costly (e.g., in computation or privacy budget), though our framework does not necessarily assume this.

As  SpiderBoost queries $\oracle_1$ every $q$ steps, the error on the estimation may accumulate and $\|\Delta_t-\nabla f(x_t)\|$ can be large.
Though on average, as shown in \cite{ABG+22}, these estimations can be good enough to find a private first-order stationary point, such a large deviation makes it challenging to analyze the behavior near a saddle point and to provide a tight analysis of the population risk. 

In our framework, rather than using $\oracle_1$ once every $q$ steps, we introduce a new technique of keeping track of the total drift we make, i.e., $\drift_t=\sum_{i=\tau_t}^t\|x_{i}-x_{i-1}\|_2^2$, where $\tau_t$ is the last time stamp when we used  $\oracle_1$.
As we are considering smooth functions, the worst error to estimate $\nabla f(x_t)-\nabla f(x_{t-1})$ is proportional to $\|x_t-x_{t-1}\|_2$.
When the $\drift_t$ is small, we know the current estimation should still be good enough, and we do not need to get an expensive fresh estimation from $\oracle_1$.
When $\drift_t$ is large, the gradient estimation error may be large and we query $\oracle_1$ and get $\Delta_t=\oracle_1(x_t)$.
To control the total cost, we need an appropriate threshold to determine when the drift is large.
The smaller the threshold is, we can guarantee  more accurate estimations but may need to pay more cost for querying $\oracle_1$ more frequently.

We want to bound the total occurrences of the event that $\drift_t$ is large, which leads to querying $\oracle_1$. A crucial observation is that, if $\drift_t$ increases quickly, then the gradient norms are large and hence function values decrease quickly, which we know does not happen frequently under the standard assumption that the function is bounded.

In our framework, we assume $\oracle_1(x)$ is an unbiased estimation of $\nabla f(x)$, and $\oracle_1(x)-\nabla f(x)$ is Norm-SubGaussian (Definition~\ref{def:subgaussian}), and similarly $\oracle_2(x,y)$ is an unbiased estimation of $\nabla f(x)-\nabla f(y)$ whose error is also Norm-SubGaussian.
In the empirical case, we can simply add Gaussian noises with appropriately chosen variances to the gradients of the empirical function $\nabla\FD$ for simplicity, and one can choose a smaller batch size to reduce the computational complexity. 
In the population case, we draw samples from the dataset without replacement to avoid dependence issues, and add the Gaussian noises to the sampled gradients. Hence we only need the gradient oracle complexity to be linear in the number of samples for the population case.

% \Sewoong{We need to explain exactly how we are different from Arora et al's SpiderBoost, and how we get population guarantee that they could not.  Also, we need to mention how we are different from Arora et al's population guarantee (and the algorithm) with tree aggregation.}
% \Daogao{I am not sure if they did not make fully use of tree aggregation technique, like change the tree structure, depth or so, or this technique is not the best fit of this problem...}

\paragraph{Minimizing excess risk.} 
Our polynomial time approach relies on the Log-Sobolev Inequality (LSI) and the classic Stroock perturbation lemma.
The previous work of \cite{MASN16} shows that  if the density $\exp\big(-\beta\FD(x)-r(x)\big)$ satisfies the LSI for some regularizer $r$, then sampling a model $x$ from  this density satisfies  differential privacy with an appropriate $(\varepsilon,\delta)$. 
If $r$ is a $\mu$ strongly convex function, then the density proportional to $\exp(-r)$ satisfies LSI with constant $1/\mu$, and $\exp(-\beta \FD(x)-r(x))$ satisfies LSI with constant $\exp(\max_{x,y}|\FD(x)-\FD(y)|)/\mu$ by the Stroock perturbation lemma.
 Our bound on the empirical risk follows from choosing the appropriate inverse temperature $\beta$ and regularizer $r$  to satisfy $(\varepsilon,\delta)$-DP. The final bound on the  population risk also follows  from LSI, which bounds the stability of the sample drawn from the respective distribution. 
%This extends to the population risk 
%LSI can also provide a generalization error bound and it is known how to sample from density satisfying LSI.

When running time is not concerned, we apply an exponential mechanism over a discretization of $\calK$ to get the upper bound. The empirical risk bound follows from \cite{BST14}, and we use concentration of sums of bounded random variables to bound the maximum difference over the discretizations between the empirical and population risk. We show this is nearly tight by reductions from selection to non-convex Lipschitz optimization of \cite{GTU22}.
%\Daogao{@Arun could you please help refine this paragraph?}

\subsection{Further related work}
In the convex setting, it is feasible to achieve efficient algorithms with good risk guarantees. In turn, differentially private empirical risk minimization (DP-ERM)~\citep{CM08,chaudhuri2011differentially, ChourasiaYS21, iyengar2019towards, kifer2012private,BST14, TTZ15,song2013stochastic, SSTT21} and differentially  private stochastic optimization~\citep{asi2021adapting, bassily2019private,bassily2020stability,feldman2019private,kulkarni2021private,AFKT21,KLZ22,GLL22,GTU22,CJJ+23,GLL+23} have been two of the most extensively studied problems in the DP literature. Most common approaches are variants of DP-SGD \citep{chaudhuri2011differentially} or the exponential mechanism \citep{mcsherrytalwar}.

As for the non-convex optimization, due to the intrinsic challenges in minimizing general non-convex functions,  most of the previous works \citep{WYX17,WJEG19,WX19,WCX19,ZCH+20,SSTT21,TC22,YZCL22,ABG+22,WB23,GW23} adopted the gradient norm as the accuracy metric rather than risk. 
Instead of  minimizing the gradient norm discussed before, \cite{BGM21} tried to minimize the stationarity gap of the population function privately, which is defined as
$
    \mathrm{Gap}_{\FP}(x):=\max_{y\in \calK}\langle \nabla \FP(x), x-y \rangle,
$
which requires $\calK$ to be a bounded domain.
There are also some different definitions of the second order stationary point. We refer the readers to \citep{LRY+20} for more details.
Some more detailed comparisons on FOSP and SOSP in the DP literature can be found in Appendix~\ref{app:related}.

The risk bound achieved by algorithms with polynomial running time is weak and requires $n\gg d$ to be meaningful.
Many previous works consider minimizing risks of non-convex functions under stronger assumptions, such as, Polyak-Lojasiewicz condition \citep{WYX17,ZLMX21}, Generalized linear model (GLM) \citep{WCX19} and weakly convex functions \citep{BGM21}. 

In the (non-private) classic stochastic optimization, there is a long line of influential works on finding the first and second-order stationary points for non-convex functions, \cite{AZB+17,JGN+17,FLLZ18,XJY18,CO19}.

% Non-DP: \cite{XJY18} proposes an algorithm to find the eigenvector for Hessian, and can be adopted to first-order algorithm. $O(d/\alpha^{3.5})$.

\section{Preliminary}
\label{sec:preliminary} 

Throughout the paper, if not stated explicitly, the norm $\|\cdot\|$ means the $\ell_2$ norm.
\begin{definition}[Lipschitz and Smoothness]
Given a function $f:\calK\to\R$, we say $f$ is $G$-Lipschitz, if for all $x_1,x_2\in\calK$,
$    |f(x_1)-f(x_2)|\le G\|x_1-x_2\|$,
and we say a function $f:\calK\to\R$ is $M$-smooth, if for all $x_1,x_2\in\calK$,
$\|\nabla f(x_1)-\nabla f(x_2)\|\le M\|x_1-x_2\|$.
\end{definition}

\begin{definition}[SubGaussian, and Norm-SubGaussian]\label{def:subgaussian}
A random vector $x\in\R^d$ is SubGaussian ($\SG(\zeta)$) if there exists a positive constant $\zeta$ such that $
    \E e^{\langle v,x-\E x\rangle}\le e^{\|v\|^2\zeta^2/2},\;\;\forall v\in\R^d$.
$x\in\R^d$ is norm-SubGaussian ($\nSG(\zeta)$) if there exists $\zeta$ such that $
    \Pr[\|x-\E x\|\ge t]\le 2e^{-\frac{t^2}{2\zeta^2}},\forall t\in\R$.
\end{definition}

\begin{fact}
\label{fact:Gaussian}
For a Gaussian $\theta\sim\calN(0,\sigma^2I_d)$, $\theta$ is $\SG(\sigma)$ and $\nSG(\sigma\sqrt{d})$.
\end{fact}

\begin{lemma}[Hoeffding type inequality for norm-subGaussian, \cite{JNG+19}]
\label{lem:concentration_nSG}
Let $x_1,\cdots,x_k\in\R^d$ be random vectors, and for each $i\in[k]$, $x_i\mid\calF_{i-1}$ is zero-mean $\nSG(\zeta_i)$ where $\calF_i$ is the corresponding filtration.
Then there exists an absolute constant $c$ such that for any $\delta>0$, with probability at least $1-\omega$,
$
    \|\sum_{i=1}^{k}x_i\|\le c\cdot\sqrt{\sum_{i=1}^{k}\zeta_i^2\log(2d/\omega)}$,
which means $\sum_{i=1}^{k}x_i$ is $\nSG(\sqrt{c\log (d)\,\sum_{i=1}^{k}\zeta_i^2})$.
\end{lemma}

\begin{definition}[Laplace distribution]
We say $X\sim\Lap(b)$ if $X$ has density $f(X=x)=\frac{1}{2b}\exp(\frac{-|x|}{b})$.
\end{definition}

\begin{theorem}[Matrix Bernstein inequality, \cite{Tro15}]
\label{thm:matrix_bern}
Consider a sequence $\{X_i\}_{i\in m}$ of independent, mean-zero, symmetric $d\times d$ random matrices.
If for each matrix $X_i$, we know $\|X_i\|_{op}\le M$, then for all $t\ge 0$, we have
$
    \Pr\left[\|\sum_{i\in[m]}X_i\|_{op}\ge t\right]\le d\exp\left(\frac{-t^2}{2(\sigma^2+Mt/3)} \right)$,
where $\sigma^2=\|\sum_{i\in[m]}\E X_i^2\|_{op}$.
\end{theorem}

\begin{theorem}[Basic composition, \cite{DR14}]
\label{thm:bas_comp}
If $\alg_1$ is $(\epsilon_1,\delta_1)$-DP and $\alg_2$ is $(\epsilon_2,\delta_2)$-DP, then their combination is $(\epsilon_1+\epsilon_2,\delta_1+\delta_2)$-DP.
\end{theorem}

\begin{theorem}[Advanced composition, \cite{kairouz15}]
\label{thm:adv_comp}
    For $\varepsilon\leq0.9$, 
    an end-to-end guarantee of $(\varepsilon,\delta)$-differential privacy is satisfied if a database  is accessed at most $k$ times, where each time with  a $(\varepsilon/(2\sqrt{2k\log(2/\delta)}),\delta/(2k))$-differentially private mechanism. 
    %a $k$-fold composition of $(\varepsilon,\delta)$-differentially private mechanisms satisfies $(k\varepsilon^2 + \varepsilon\sqrt{2k\log(1/k\delta)} ,2k\delta)$-differential privacy. 
\end{theorem}
Due to space limit, some proofs are left in the Appendix.
\section{Convergence to Stationary points}
\label{sec:stationary_points}
We follow the assumptions of  \cite{WCX19}, which also studies privately finding an $\alpha$-SOSP. 
\begin{assumption}
\label{assump:SOSP}
 Any function drawn from $\calP$ is $G$-Lipschitz, $\rho$-Hessian Lipschitz, and $M$-smooth, almost surely, and the risk is upper bounded by $B$. 
\end{assumption}

As discussed before, we define two different kinds of gradient oracles, one for estimating the gradient at one point and the other for estimating the gradient difference at two points.
%Now we define them formally.
\begin{definition}[SubGaussian gradient oracles]
\label{def:oracle}
For a $G$-Lipschitz and $M$-smooth function $F$:\\
$(1)$ We say $\oracle_1$ is a first kind of $\zeta_1$ norm-subGaussian Gradient oracle if given $x\in\R^d$, $\oracle(x)$ satisfies $\E\oracle_1(x)=\nabla F(x)$ and $\oracle_1(x)-\nabla F(x)$ is $\nSG(\zeta_1)$.\\
$(2)$ We say $\oracle_2$ is a second kind of $\zeta_2$ norm-subGaussian stochastic Gradient oracle if given $x,y\in\R^d$, $\oracle_2(x,y)$ satisfies that $\E\oracle_2(x,y)=\nabla F(x)-\nabla F(y)$ and $\oracle_2(x,y)-(\nabla F(x)-\nabla F(y))$ is $\nSG(\zeta_2\|x-y\|)$.
\end{definition}

Note that we should assume $M\ge \sqrt{\rho\alpha}$ to make finding  a second-order stationary point strictly more challenging than finding a  first-order stationary point.
We use $\smin(\cdot)$ to denote the smallest eigenvalue of a matrix.

\subsection{Meta framework}
\begin{algorithm}[htb]
\caption{Stochastic Spider}
\label{alg:DP-Spider-META}
\begin{algorithmic}[1]
\STATE {\bf Input:} Objective function $F$, Gradient Oracle $\oracle_1,\oracle_2$ with SubGaussian parameters $\zeta_1$ and $\zeta_2$, parameters of objective function $B,M,G,\rho$, parameter $\kappa$, failure probability $\omega$
\STATE Set $\gamma=\sqrt{4C(\zeta_2^2\kappa+4\zeta_1^2)\cdot\log(BMd/\rho\omega)},\Gamma=\frac{M\log(\frac{dMB}{\rho\gamma\omega})}{\sqrt{\rho\gamma}}$
\STATE Set $\eta=1/M,t=0,T=BM\log^4(\frac{dMB}{\rho \gamma\omega})/\gamma^2$
\STATE Set $\drift_0=\kappa,\frozen=1,\nabla_{-1}=0$
\WHILE{$t\le T$}
\IF{$\|\nabla_{t-1}\|\le \gamma\log^3(BMd/\rho\omega) \bigwedge\frozen_{t-1}\le 0$}
{
\STATE $\frozen_t= \Gamma,\drift_t= 0$
\STATE $\nabla_t=\oracle_1(x_t)+g_t$, where $g_t\sim\calN(0,\frac{\zeta_1^2}{d} I_d)$
}
% \ELSIF{ $\mathrm{mod}(t,q)=0\bigvee\drift\ge \gamma$}
\ELSIF{$\drift_{t-1}\ge \kappa$}
\STATE $\nabla_t=\oracle_1(x_t)$, 
 $\drift_t=0$, 
 $\frozen_t=\frozen_{t-1}-1$
\ELSE
\STATE $\Delta_t=\oracle_2(x_t,x_{t-1})$,  $\nabla_t=\nabla_{t-1}+\Delta_t$, $\frozen_t=\frozen_{t-1}-1$
\ENDIF
\STATE $x_{t+1}=x_{t}-\eta\nabla_t,\drift_t=\drift_{t-1}+\eta^2\|\nabla_t\|_2^2$, $t=t+1$
\ENDWHILE
\STATE {\bf Return:} $\{x_1,\cdots,x_T\}$
\end{algorithmic}
\end{algorithm}

We demonstrate a framework based on the SpiderBoost  in Algorithm~\ref{alg:DP-Spider-META}.
Our analysis of Algorithm~\ref{alg:DP-Spider-META} builds upon three key properties we prove in this section:
($i$) $\nabla_t$ is consistently close to the true gradient $\nabla F(x_t)$ with high probability;
($ii$) the algorithm can escape the saddle point with high probability, and  
($iii$) a large $\drift$ implies significant decrease in the function value, allowing us to limit the number of queries to the more accurate but more expensive first kind of gradient oracle $\oracle_1$. %can only occur a bounded number of  its frequency can be bounded. %the total number of times for drift to be large is bounded.

\begin{restatable}{lem}{GoodGradientEstimator}
\label{lm:good_gradient_estimator}
For any $0\le t\le T$ and letting $\tau_t\le t$ be the largest integer such that $\drift_{\tau_t}$ is set to be 0, with probability at least $1-\omega/T$, for some universal constant $C>0$, we have
\begin{align}
\label{eq:graident_error}
    \|\nabla_t-\nabla F(x_{t})\|^2\le  \big(\zeta_2^2\cdot \sum_{i=\tau_t+1}^{t}\|x_{i}-x_{i-1}\|^2+4\zeta_1^2\big)\cdot C\cdot\log(Td/\omega).
\end{align}
Hence with probability at least $1-\omega$, we know for each $t\le T$, $
    \|\nabla_t-\nabla F(x_t)\|^2\le\gamma^2/16$,
where 
$\gamma^2:=16C(\zeta_2^2\kappa+4\zeta_1^2)\cdot\log(Td/\omega)$ and $\kappa$ is a parameter we can choose in the algorithm.
\end{restatable}

As shown in Lemma~\ref{lm:good_gradient_estimator}, the error on the gradient estimation for each step is bounded with high probability. 
Then we can show the algorithm can escape the saddle point efficiently based on previous results.

\begin{restatable}[Essentially from \cite{WCX19}]{lem}{EscapeSaddlePoint}
\label{lem:escape_saddle_point}
Under Assumption~\ref{assump:SOSP}, run SGD iterations $x_{t+1}=x_t-\eta \nabla_t$, with step size $\eta=1/M$. 
Suppose $x_0$ is a stationary point satisfying  $\|\nabla F(x_0)\|\le \alpha$ and $\smin(\nabla^2 F(x_0))\le -\sqrt{\rho\alpha}$, $\alpha=\gamma\log^3(dBM/\rho\omega)$. 
If $\nabla_0=\nabla F(x_0)+\zeta_1+\zeta_2$ where $\|\zeta_1\|\le \gamma$, $\zeta_2\sim\calN(0,\frac{\gamma^2}{d\log(d/\omega)} I_d)$,  and $\|\nabla_t-\nabla F(x_t)\|\le \gamma$ for all $t\in[\Gamma]$,  with probability at least $1-\omega\cdot\log(1/\omega)$, one has
\begin{align*}
    F(x_{\Gamma})-F(x_0)\le -\Omega\big(\frac{\gamma^{3/2}}{\sqrt{\rho}\log^3(\frac{dMB}{\rho\gamma\omega})}\big),
\end{align*}
where $\Gamma=\frac{M\log(\frac{dMB}{\rho\gamma\omega})}{\sqrt{\rho\gamma}}$.
\end{restatable}
We discuss this lemma in the Appendix~\ref{subsec:proof_escape} in more details.
The next lemma is standard, showing how large the function values can decrease in each step.

\begin{restatable}{lem}{ValueDecrease}
\label{lm:value_decrease}
By setting $\eta=1/M$, we have
\begin{align*}
    F(x_{t+1})\le F(x_t)+\eta \|\nabla_t\|\cdot\|\nabla F(x_t)-\nabla_t\|-\frac{\eta}{2}\|\nabla_t\|^2.
\end{align*}
Moreover, with probability at least $1-\omega$, for each $t\le T$ such that $\|\nabla F(x_t)\|\ge \gamma$, we have
\begin{align*}
    F(x_{t+1})-F(x_t)\le -\eta \|\nabla_t\|^2/6\le -\eta \gamma^2/6.
\end{align*}
\end{restatable}

With the algorithm designed to control the $\drift$ term, the guarantee for Stochastic Spider to find the second order stationary point is stated below:
\begin{restatable}{lem}{GPS}
\label{lem:guarantee_private_spider}
Suppose $\oracle_1$ and $\oracle_2$ are $\zeta_1$ and $\zeta_2$ norm-subGaussian respectively.
If one sets $\gamma=O(1)\sqrt{(\zeta_2^2\kappa+4\zeta_1^2)\cdot\log(Td/\omega)}$, with probability at least $1-\omega$, at least one point in the output set $\{x_1,\cdots,x_T\}$ of Algorithm~\ref{alg:DP-Spider-META} is $\alpha$-SOSP, where
\begin{align*}
    \alpha=\gamma \log^3(BMd/\rho\omega\gamma)=\sqrt{(\zeta_2^2\kappa+4\zeta_1^2)\cdot\log(\frac{d/\omega}{\zeta_2^2\kappa+\zeta_1^2})}\cdot\log^3(\frac{BMd}{\rho\omega(\zeta_2^2\kappa+\zeta_1^2)}).
\end{align*}
\end{restatable}

As mentioned before, we can bound the number of occurrences where the $\drift$ gets large and hence bound the total time we query the oracle of the first kind. 

\begin{restatable}{lem}{BoundTotalTimesLargeDrift}
\label{lm:bound_total_times_large_drift}
Under the event that $\|\nabla_t-\nabla F(x_t)\|\le\gamma/4$ for all $t\in[T]$ and our parameter settings,
letting $K=\{t\in[T]:\drift_t\ge \kappa\}$ be the set of iterations where the drift is large, we know $|K|\le O\big(\frac{B\eta}{\kappa}+T\gamma^2\eta^2/\kappa)=O(B\eta\log^4(\frac{dMB}{\rho\gamma\omega})/\kappa\big)$.
\end{restatable}

\subsection{Convergence to the SOSP of the empirical risk}
We use Stochastic Spider  to improve the convergence to $\alpha$-SOSP of the empirical risk, and aim at getting $\alpha=\Tilde{O}({d^{1/3}}/{n^{2/3}})$.
We let $\FD$ be the objective function $F$ and use the gradient oracles 
\begin{align}
\label{eq:oracle_empirical}
    \oracle_1(x):=\nabla\FD(x)+g_1, \text{ and }\;\;\;
    \oracle_2(x,y) := \nabla\FD(x)-\nabla\FD(y)+g_2,
\end{align} where $g_1\sim\calN(0,\sigma_1^2I_d)$ and $g_2\sim\calN(0,\sigma_2^2I_d)$ ensures privacy. 

Before stating the formal results, note that by Lemma~\ref{lem:guarantee_private_spider}, the framework can only guarantee the existence of an $\alpha$-SOSP in the outputted set.
In order to find the SOSP privately from the set, we adopt the well-known AboveThreshold algorithm, whose pseudo-code can be found in Algorithm~\ref{alg:privately_select}.
\begin{algorithm}[ht]
\begin{algorithmic}[1]
\caption{AboveThreshold}
\label{alg:privately_select}
\STATE {\bf Input:} A set of points $\{x_i\}_{i=1}^T$, dataset $S$, parameters of objective function $B,M,G,\rho$, objective error $\alpha$
\STATE Set $\hat{T}_1=\alpha+\Lap(\frac{4G}{n\epsilon})+\frac{16\log(2T/\omega)G}{n\epsilon},\hT_2=-\sqrt{\rho\alpha}+\Lap(\frac{4M}{n\epsilon})-\frac{16\log(2T/\omega)M}{n\epsilon}$ 
\FOR{$i=1,\cdots,T$}
\IF{$\|\nabla F_S(x_i)\|+ \Lap(\frac{8G}{n\epsilon})\le \hT_1 \bigwedge \smin(\nabla^2 F_S(x_i))+\Lap(\frac{8M}{n\epsilon})\ge\hT_2 $}
\STATE {\bf Output:} $x_i$
\STATE {\bf Halt}
\ENDIF
\ENDFOR
\end{algorithmic}
\end{algorithm}
Algorithm~\ref{alg:privately_select} is a slight modification of the AboveThreshold algorithm \cite{DR14}, and we get the following guarantee immediately.
\begin{lemma}
\label{lem:gurantee_of_abovethreshold}
Algorithm~\ref{alg:privately_select} is $(\epsilon,0)$-DP.
Given the point set $\{x_1,\cdots,x_T\}$ and $S$ of size $n$ as the input,
\begin{itemize}
    \item if it outputs any point $x_i$, then with probability at least $1-\omega$, we know 
    \begin{align*}
        \|\nabla F_S(x_i)\|\le \alpha+\frac{32\log(2T/\omega)G}{n\epsilon}, \text{ and } \smin(\nabla^2 F_S(x_i))\ge -\sqrt{\rho\alpha}-\frac{32\log(2T/\omega)M}{n\epsilon}
    \end{align*}
    \item if there exists a $\alpha$-SOSP point $x\in \{x_i\}_{i\in[T]}$, then with probability at least $1-\omega$, Algorithm~\ref{alg:privately_select} will output one point.
\end{itemize}
% if there exists a point $x\in S$ such that $\|\nabla\FD(x)\|\le \alpha$ and $\smin(\nabla^2 \FD(x))\ge -\sqrt{\rho\alpha}$
\end{lemma}

Combining Algorithm~\ref{alg:DP-Spider-META} and Algorithm~\ref{alg:privately_select}, we can find the SOSP we want, which is stated formally below:
\begin{restatable}[Empirical]{theorem}{SospEmp}
\label{thm:sosp_emp}
Using full batch in Algorithm~\ref{alg:DP-Spider-META}, 
and setting $\kappa=\frac{G^{4/3}B^{1/3}}{M^{5/3}}(\frac{\sqrt{d\log(1/\delta)}}{n\epsilon})^{2/3}$, $\sigma_1=\frac{G\sqrt{B\eta\log^2(1/\delta)/\kappa}\log^2(ndMB/\omega)}{n\epsilon},\sigma_2=\frac{M\sqrt{\log^2(1/\delta)BM/\alpha_1^2}\log^5(ndMB/\omega)}{n\epsilon}$,
Algorithm~\ref{alg:DP-Spider-META} is $(\epsilon,\delta)$-DP, and with probability at least $1-\omega$, at least one point in the output set $\{x_i\}_{i\in[T]}$ is $\alpha_1$-SOSP of $\FD$ with
\begin{align*}
    \alpha_1=O\left(\big(\frac{\sqrt{dBGM\log^2(1/\delta)}}{n\epsilon}\big)^{2/3}\cdot\log^6\frac{nBMd}{\rho\omega} \right).
\end{align*}

Moreover, if we run Algorithm~\ref{alg:privately_select} with inputs $\{x_i\}_{i\in[T]}, \calD, B,M,G,\rho,\alpha_1$, with probability at least $1-\omega$, we can get an $\alpha_2$-SOSP of $\FD$ with
\begin{align*}
    \alpha_2=O\left(\alpha_1+\frac{G\log(n/G\omega)}{n\epsilon}+\frac{M\log(ndBGM/\rho\omega)}{n\epsilon\sqrt{\rho}}\sqrt{\alpha_1}\right).
\end{align*}
\end{restatable}

\subsection{Convergence to the SOSP of the population risk}
This subsection aims at getting an $\alpha$-SOSP for $\FP$ (the population function). Differing from the stochastic oracles used for empirical function $\FD$, we do not use full batch in the oracle.
As an alternative, we draw fresh samples from $\calD$ without replacement with a smaller batch size:
\begin{align}
    \label{eq:oracle_population}
    \oracle_1(x):=\frac{1}{b_1}\sum_{z\in S_1}\nabla f(x;z)+g_1,\text{ and }
    \oracle_2(x,y):=\frac{1}{b_2}\sum_{z\in S_2}(\nabla f(x;z)-\nabla f(y;z))+g_2,
    \end{align}
where $S_1$ and $S_2$ are sets of size of $b_1$ and $b_2$ respectively drawn from $\calD$ without replacement, $g_1\sim\calN(0,\sigma_1^2I_d)$ and $g_2\sim\calN(0,\sigma_2^2\|x-y\|_2^2\cdot I_d)$.
These gradient oracles satisfy the following. 
\begin{claim}
The gradient oracles $\oracle_1$ and $\oracle_2$ constructed in Equation~\eqref{eq:oracle_population} are a first kind of $O(\frac{L\sqrt{\log d}}{\sqrt{b_1}}+\sqrt{d}\sigma_1)$ norm-subGaussian gradient oracle and second kind of $O(\frac{M\sqrt{\log d}}{\sqrt{b_2}}+\sqrt{d}\sigma_2)$ norm-subGaussian gradient oracle respectively.
\end{claim}

\begin{proof}
For the oracle $\oracle_1$, 
 we know for each $z\in S_1$, $\E_{z\sim\calP}[\nabla f(x,z)]=\nabla \FP(x)$ and $\nabla f(x,z)-\nabla \FP(x)$ is $\nSG(L)$ due to the Lipschitzness assumption.
The statement follows from Fact~\ref{fact:Gaussian} and Lemma~\ref{lem:concentration_nSG}. 
As for the $\oracle_2$, the statement follows similarly with the smoothness assumption.
\end{proof}

Recall that in the empirical case, we use Algorithm~\ref{alg:privately_select} to choose the SOSP for $\FD$.
But in the population case, we need to find SOSP for $\FP$, and what we have are samples from $\calP$.
We need the following technical results to help us find the SOSP from the set, which follows from  Hoeffding inequality for norm-subGaussians (Lemma~\ref{lem:concentration_nSG}) and Matrix Bernstein inequality (Theorem~\ref{thm:matrix_bern}).

\begin{restatable}{lem}{EmpPopSosp}
\label{lem:emp_pop_sosp}
Fix a point $x\in\R^d$. 
Given a set $S$ of $m$ samples drawn i.i.d. from the distribution $\calP$, then we know with probability at least $1-\omega$, we have
\begin{align*}
    \|\nabla F_{S}(x)-\nabla \FP(x)\|_2\le O\big(\frac{G\log(d/\omega)}{\sqrt{m}}\big)\bigwedge \|\nabla^2F_S(x)-\nabla^2\FP(x)\|_{op}\le O\big(\frac{M\log(d/\omega)}{\sqrt{m}}\big).
\end{align*}
\end{restatable}
We can bound the population bound similar to the empirical bound with these tools.
\begin{restatable}[Population]{theorem}{SospPop}
\label{thm:sosp_pop}
Divide the dataset $\calD$ into two disjoint datasets $\calD_1$ and $\calD_2$ of size $\lceil n/2\rceil$ and $\lfloor n/2\rfloor$ respectively.
Setting $b_1=\frac{n\kappa}{B\eta},b_2=\frac{n\alpha_1^2}{BM},\sigma_1=\frac{G\sqrt{\log(1/\delta)}}{b_1\epsilon},\sigma_2=\frac{M\sqrt{\log(1/\delta)}}{b_2\epsilon}$ and $\kappa=\max (\frac{G^{4/3}B^{1/3}\log^{1/3}d}{M^{5/3}}n^{-1/3},(\frac{GB^{2/3}}{M^{5/3}})^{6/7} (\frac{\sqrt{d\log(1/\delta)}}{n\epsilon})^{4/7})$ in Equation~\eqref{eq:oracle_population} and using them as gradient oracles, Algorithm~\ref{alg:DP-Spider-META} with $\calD_1$ is $(\epsilon,\delta)$-DP, and with probability at least $1-\omega$, at least one point in the output is $\alpha_1$-SOSP of $\FP$ with
    \begin{align*}
    \alpha_1=O\Big(\big((BGM\cdot\log d)^{1/3}\frac{1}{n^{1/3}}+(G^{1/7}B^{3/7}M^{3/7}) (\frac{\sqrt{d\log(1/\delta)}}{n\epsilon})^{3/7}\big)\log^3(nBMd/\rho\omega)\Big).
    \end{align*}
    
    Moreover, if we run Algorithm~\ref{alg:privately_select} with inputs $\{x_i\}_{i\in[T]},\calD_2,B,M,G,\rho,\alpha_1$, with probability at least $1-\omega$, Algorithm~\ref{alg:privately_select} can output an $\alpha_2$-SOSP of $\FP$ with
    \begin{align*}
        \alpha_2=O\left(\alpha_1+\frac{M\log(ndBGM/\rho\omega)}{\sqrt{\rho}\min(n \epsilon, n^{1/2})}\sqrt{\alpha_1}+G(\frac{\log(n/G\omega)}{n\epsilon}+\frac{\log(d/\omega)}{\sqrt{n}})\right).
        % (BGM)^{1/6}(\frac{\sqrt{d\log(1/\delta)}}{n\epsilon})^{1/3}+\frac{1}{n^{1/6}})\right).
    \end{align*}
\end{restatable}

% \begin{restatable}[Population]{theorem}{SospPop}
% \label{thm:sosp_pop}
% Divide the dataset $\calD$ into two disjoint datasets $\calD_1$ and $\calD_2$ of size $\lceil n/2\rceil$ and $\lfloor n/2\rfloor$ respectively.
%     Setting $b_1=\frac{n\kappa}{B\eta},b_2=\frac{n\alpha_1^2}{BM},\sigma_1=\frac{G\sqrt{B\eta\log(1/\delta)/\kappa}}{n\epsilon},\sigma_2=\frac{M\sqrt{\log(1/\delta)BM/\alpha_1^2}}{n\epsilon}$ and $\kappa=\frac{G^{4/3}B^{1/3}}{M^{5/3}}\max\{(\frac{\sqrt{d\log(1/\delta)}}{n\epsilon})^{2/3},n^{-1/3}\}$ in Equation~\eqref{eq:oracle_population} and using them as gradient oracles, Algorithm~\ref{alg:DP-Spider-META} with $\calD_1$ is $(\epsilon,\delta)$-DP, and with probability at least $1-\omega$, at least one point in the output is $\alpha_1$-SOSP of $\FP$ with
%     \begin{align*}
%         \alpha_1=O\Big(\log^3\frac{nBMd}{\rho\omega}(BGM)^{1/3}\big((\frac{\sqrt{d\log(1/\delta)}}{n\epsilon})^{2/3}+\frac{1}{n^{1/3}}\big)\Big).
%     \end{align*}
    
%     Moreover, if we run Algorithm~\ref{alg:privately_select} with inputs $\{x_i\}_{i\in[T]},\calD_2,B,M,G,\rho,\alpha_1$, with probability at least $1-\omega$, Algorithm~\ref{alg:privately_select} can output an $\alpha_2$-SOSP of $\FP$ with
%     \begin{align*}
%         \alpha_2=O\left(\alpha_1+\frac{M\log(ndBGM/\rho\omega)}{\min(\sqrt{\rho n \epsilon}, n^{1/4})}\sqrt{\alpha_1}+G(\frac{\log(n/G\omega)}{n\epsilon}+\frac{\log(d/\omega)}{\sqrt{n}})\right).
%         % (BGM)^{1/6}(\frac{\sqrt{d\log(1/\delta)}}{n\epsilon})^{1/3}+\frac{1}{n^{1/6}})\right).
%     \end{align*}
% \end{restatable}
\section{Bounding the excess risk}
\label{sec:risk_bound}
In this section, we consider the risk bounds.

\subsection{Polynomial time approach}
% \Daogao{I think there should be NP hardness results in dimension $d$ for this problem.}

If we want the algorithm to be efficient and implementable in polynomial time, to our knowledge the only known bound is $ O(\frac{d\log(1/\delta)}{\epsilon^2\log n})$ in \cite{WCX19} for smooth functions.
\cite{WCX19} used Gradient Langevin Dynamics, a popular variant of SGD to solve this problem, and prove the privacy by advanced composition.
We generalize the exponential mechanism to the non-convex case and implement it without smoothness assumption.

First recall the Log-Sobolev inequality:
We say a probability distribution $\pi$ satisfies  LSI with constant $\CLSI$ if for all $f:\R^d\to\R$,
$    \E_{\pi}[f^2\log f^2]-\E_\pi[f^2]\log\E_{\pi}[f^2]\le 2\CLSI \E_\pi\|\nabla f\|_2^2.$

A well-known result (\cite{OV00}) says if $f$ is $\mu$-strongly convex, then the distribution proptional to $\exp(-f)$ satisfies LSI with constant $1/\mu$.
Recall the results from previous results \cite{MASN16} about LSI and DP:
\begin{theorem}[\cite{MASN16}]
\label{thm:LSI_to_DP}
Sampling from $\exp(-\beta F(x;\calD)-r(x))$ for some public regularizer $r$ is $(\epsilon,\delta)$-DP, where
$
    \epsilon\le 2\frac{G\beta}{n}\sqrt{\CLSI}\sqrt{1+2\log(1/\delta)},
$
and $\CLSI$ is the worst LSI constant.
\end{theorem}

% \subsubsection{Empirical Risk by Perturbation Argument}
We can apply the classic perturbation lemma to get the new LSI constant in the non-convex case.
Suppose we add a regularizer $\frac{\mu}{2}\|x\|^2$, and try to sample from $\exp(-\beta(F(x;\calD)+\frac{\mu}{2}\|x\|^2))$.

\begin{lemma}[Stroock perturbation]
\label{lm:stroock_pert}
Suppose $\pi$ satisfies LSI with constant $\CLSI(\pi)$.
If $0<c\le \frac{\d \pi'}{\d \pi}\le C $,
then $\CLSI(\pi')\le \frac{C}{c}\CLSI(\pi)$.
\end{lemma}
Lemma~\ref{lem:utility_nonconvex_sampling} is a more general version of Theorem 3.4 in \cite{GTU22} and can be used to bound the empirical risk.
% \iffalse
% We have the following perturbation result, which can be derived from Holley–Stroock perturbation.
% \begin{lemma}[\cite{}]
% \label{lm:perturbation}
% Suppose a function $W:\R^d\to \R$ can be written as $W=W_c+W_l$, where $W_c$ is $\mu$-strongly convex and $W_l$ is $G$-Lipschitz, then measure proportional to $\exp(-W)$ satisfies LSI with constant 
% \begin{align*}
%     \CLSI\le \frac{4}{\mu}\exp(\frac{4G^2\sqrt{d}}{\mu}),
% \end{align*}
% \end{lemma}

% Combining Lemma~\ref{lm:perturbation} and Theorem~\ref{thm:LSI_to_DP}, can get empirical risk $O(\frac{d^{3/4}}{\sqrt{\log n}})$ by setting $\beta=\frac{\mu \log n}{G^2\sqrt{d}}$ and $\mu=\frac{Gd^{3/4}}{D\sqrt{\log n}}$, which is worse than \cite{WCX19}, and exponential mechanism, as $\beta=\frac{\sqrt{\log (n)} d^{1/4}}{GD}$...
% The regularized exponential mechanism works well in the convex case is that one can set $\beta \gg \frac{n}{GD}$.
% \fi

\begin{lemma}
\label{lem:utility_nonconvex_sampling}
Let $\pi(x)\propto \exp(-\beta(\FD(x)+\frac{\mu}{2}\|x\|_2^2))$. Then for $\beta GD>d$, we know
\begin{align*}
    \E_{x\sim\pi}(\FD(x)+\frac{\mu}{2}\|x\|_2^2)-\min_{x^*\in\calK}(\FD(x^*)+\frac{\mu}{2}\|x^*\|_2^2)\le \frac{d}{\beta}\log(\beta G D / d)
\end{align*}
\end{lemma}

% \begin{proof}
% We let $\Lambda:=\int \pi \d x$ be the normalization constant, and $\mathcal{E}(\pi):=-\int \pi\log\pi\d x$ be the entropy.
% Then we have
% \begin{align*}
%     \E_{x\sim\pi}(\FD(x)+\frac{\mu}{2}\|x\|_2^2)=\frac{1}{\beta}(\mathcal{E}(\pi)-\log\Lambda).
% \end{align*}
% By the Equation (3.20) in \cite{RRT17}, we can bound the entropy,
% \begin{align*}
%     \mathcal{E}(\pi)\le \frac{d}{2}\log\left(\frac{30}{\mu}(\frac{G^2}{2\mu}+\frac{d}{\beta})\right).
% \end{align*}
% It suffices to bound $\Lambda$.
% Let $\HF= \FD(x)+\frac{\mu}{2}\|x\|_2^2$, $x^*=\arg\min_{x\in\calK}\FD(x)+\frac{\mu}{2}\|x\|_2^2$ and $\HF^*=\HF(x^*)$.
% For any $x$, we know $\HF(x)-\HF^*\le G\|x-x^*\|_2+\frac{\mu}{2}(\|x\|_2^2-\|x^*\|_2^2)\le G\|x-x^*\|_2+\mu D^2/2$.
% Hence we have
% \begin{align*}
%     \log\Lambda=&\log \int e^{-\beta\HF(x)}\d x\\
%     =& -\beta\HF^*+\log \int e^{\beta(\HF^*-\HF)}\d x\\
%     \ge & -\beta\HF^*+\log \int e^{-\beta G\|x-x^*\|_2-\frac{\mu\beta}{2}D^2}\d x
% \end{align*}
% \end{proof}

We now turn to bound the generalization error, and use the notion of uniform stability:
\begin{lemma}[Stability and Generalization \cite{BE02}]
\label{lem:stab_gen}
Given a dataset $\calD=\{s_i\}_{i\in[n]}$ drawn i.i.d. from some underlying distribution $\calP$, and given any algorithm $\alg$, suppose we randomly replace a sample $s$ in $\calD$ by an independent fresh one $s'$ from $\calP$ and get the neighoring dataset $\calD'$, then
$
    \E_{\calD,\alg}[\FP(\alg(\calD))-\FD(\alg(\calD))]=\E_{\calD,s',\alg}[f(\alg(\calD);s'))-f(\alg(\calD');s'))],
$
where $\alg(\calD)$ is the output of $\alg$ with input $\calD$.
\end{lemma}

As each function $f(;s')$ is $G$-Lipschitz, it suffices to bound the $W_2$ distance of $\alg(\calD)$ and $\alg(\calD')$.
If $\alg$ is sampling from the exponential mechanism, letting $\pi_{\calD}\propto \exp(-\beta(\FD(x)+\frac{\mu}{2}\|x\|^2))$ and $\pi_{\calD'}\propto \exp(-\beta(F_{\calD'}(x)+\frac{\mu}{2}\|x\|^2))$, it suffices to bound the $W_2$ distance between $\pi_{\calD}$ and $\pi_{\calD'}$.
The following lemma can bound the generalization risk of the exponential mechanism under LSI:
\begin{restatable}[Generalization error bound]{lem}{GenError}
\label{lem:gen_error}
Let $\pi_{\calD}\propto\exp(-\beta(\FD(x)+\frac{\mu}{2}\|x\|_2^2))$.
Then we have
\begin{align*}
    \E_{\calD,x\sim\pi_{\calD}}[\FP(x)-\FD(x)]\le O(\frac{G^2\exp(\beta GD)}{n\mu}).
\end{align*}
\end{restatable}

We get the following results:
\begin{restatable}[Risk bound]{theorem}{ZerothPolytimeEmpirical}
\label{thm:zeroth_polytime_empirical}
We are given $\epsilon,\delta\in(0,1/2)$.
Sampling from $\exp(-\beta(\FD(x)+\frac{\mu}{2}\|x\|_2^2))$ with $\beta=O(\frac{\epsilon\log(nd)}{GD\sqrt{\log(1/\delta))}}),\mu=\frac{d}{D^2\beta}$ is $(\epsilon,\delta)$-DP.
The empirical risk and population risk are bounded by
$
    O(GD\frac{d\cdot\log\log(n)\sqrt{\log(1/\delta)}}{\epsilon\log(nd)})$.
\end{restatable}

\subsubsection{Implementation}
There are multiple existing algorithms that can sample efficiently from density with LSI, under mild assumptions.
For example, when the functions are smooth or weakly smooth, one can turn to the Langevin Monte Carlo \cite{CEL+21}, and \cite{LC22}.
The algorithm in \cite{WCX19} also requires mild smoothness assumptions.
We discuss the implementation of non-smooth functions in bit more details, which is more challenging.
\iffalse
\Daogao{There is some issue when functions are non-smooth. Hence we may give this subsection up and only cite some direct results, for smooth, weakly smooth and semi-smooth functions.}
% Note the $\CLSI(\pi)$ is bounded by $O(D^2 n)$.
We discuss how to implement the exponential mechanism.
If the functions are smooth, or even weakly smooth, there are sampling algorithms that can work directly under LSI, for example, the Langevin Monte Carlo \cite{CEL+21}.
\fi

We can adopt the rejection sampler in \cite{GLL22}, which is based on the alternating sampling algorithm in \cite{LST21}.
Both \cite{LST21} and \cite{GLL22} are written in the language of log-concave and strongly log-concave densities, but their results hold as long as LSI holds.
By combining them together, we can get the following risk bounds.
The details of the implementation can be found in Appendix~\ref{subsec:implementation}.
\begin{restatable}[Implementation, risk bound]{theorem}{Implementation}
\label{thm:implementation}
For $\epsilon,\delta\in(0,1/2)$, there is an $(\epsilon,2\delta)$-DP efficient sampler that can achieve the empirical and population risks 
$
    O(GD\frac{d\cdot\log\log(n)\sqrt{\log(1/\delta)}}{\epsilon\log(nd)}).
$
Moreover, in expectation, the sampler takes
$
\tilde{O}\left( {n\epsilon^3 \log^3(d)\sqrt{\log(1/\delta)}}/({GD})\right) 
$
function values query and some Gaussian random variables restricted to the convex set $\calK$ in total.
\end{restatable}

% \begin{proof}[Proof Sketch]
% The proof in \cite{LST21} relies on the LSI.
% Let $\pi_{y_t}\propto \exp(-\HF(x)-\frac{1}{2\eta}\|x-y_t\|_2^2)$.
% We know $\pi_{y_t}$ satisfies LSI with constant $\CLSI(\pi_{y_t})\le \frac{\exp(GD)}{\mu+\frac{1}{\eta}}$.
% They show the $W_2^2$ distance between the density $\pi_{y_t}$ and the stationary distribution decreases by a rate of $\CLSI(\pi_{y_t})^2/\eta^2$, that is
% \begin{align*}
%     W_2^2(\pi_{y_{t+1}},\pi)\le \frac{\CLSI(\pi_{y_t})^2}{\eta^2}W_2^2(\pi_{y_{t}},\pi)\le \frac{\exp(2GD)}{(1+\eta \mu)^2}W_2^2(\pi_{y_{t}},\pi).
% \end{align*}
% We know $W_2^2(\pi_{y_1},\pi)\le D^2$.
% When $\eta\le \frac{1}{\mu \exp(GD)}$

% \end{proof}
\subsection{Exponential time approach} 
In \cite{GTU22}, it is shown that sampling from $\exp(-\frac{\epsilon n}{GD}\FD(x))$ is $\varepsilon$-DP, and a nearly tight empirical risk bound of  $\Tilde{O}(\frac{DGd}{n\epsilon})$ is achieved for convex functions.
It is open what is the bound we can get for non-convex DP-SO.

\subsubsection{Upper Bound}
Given exponential time we can use a discrete exponential mechanism as considered in \cite{BST14}. We recap the argument and extend it to DP-SO.
The proof is based on a simple packing argument, and can be found in Appendix~\ref{subsec:proof_of_thm_risk_exp}.
\begin{restatable}{theorem}{RiskExp}
\label{thm:risk_exp}
There exists an $\epsilon$-DP differentially private algorithm that achieves a population risk of 
$O\left(GD \left( {d \log(\epsilon n / d)}/({\epsilon n}) + {\sqrt{d \log(\epsilon n / d)}}/({\sqrt{n}})\right)\right)$.
\end{restatable}

% \begin{theorem} 
% \label{thm:risk_exp}
% There exists an $\epsilon$-DP differentially private algorithm with population risk 
% \[O\left(GD \left(\frac{d \log(\epsilon n / d)}{\epsilon n} + \frac{\sqrt{d \log(\epsilon n / d)}}{\sqrt{n}}\right)\right).\]
% \end{theorem}

\subsubsection{Lower Bound}
Results in \cite{GTU22} imply that the first term of $\tilde O( GDd/\epsilon n)$ is tight, even if we relax to approximate DP with $\delta>0$.  
A reduction from private selection problem shows the $\tilde O(\sqrt{d/n})$ generalization term is also nearly-tight (Theorem~\ref{thm:risk_lb}). In the selection problem,  we have $k$ coins, each with an unknown probability $p_i$. Each coin is flipped $n$  times such that $\{x_{i,j}\}_{j \in [n]}$, each $x_{i,j}$ i.i.d. sampled from ${\rm Bern}(p_i)$, and we want to choose a coin $i$ with the smallest $p_i$. The risk of choosing $i$ is $p_i - \min_{i^*} p_{i^*}$.

\begin{theorem}
Any algorithm for the selection problem has excess population risk $\tilde{\Omega}(\sqrt{\frac{\log k}{n}})$. 
\end{theorem}

This follows from a folklore result on the selection problem (see e.g. \cite{BU17}). We can combine this with the following reduction from selection to non-convex optimization:

\begin{theorem}[Restatement of results in \cite{GTU22}]
If any $(\epsilon, \delta)$-DP algorithm for selection has risk $R(k)$, then any $(\epsilon, \delta)$-DP algorithm for minimizing 1-Lipschitz losses over $B_d(0, 1)$ (the $d$-dimensional unit ball) has risk $R(2^{\Theta(d)})$.
\end{theorem}

From this and the aforementioned lower bounds in empirical non-convex optimization we get the following:

\begin{theorem}
\label{thm:risk_lb}
For $\epsilon \leq 1, \delta \in [2^{-\Omega(n)}, 1/n^{1 + \Omega(1)}]$, any $(\epsilon, \delta)$-DP algorithm for minimizing $1$-Lipschitz losses over $B_d(0, 1)$ has excess population risk $\max \{ \Omega (d \log(1/\delta)/(\epsilon n) ), \tilde{\Omega} (\sqrt{d/n} ) \}$.
\end{theorem}
\section{Conclusion}
We present a novel framework that can improve upon the state-of-the-art rates for locating second-order stationary points for both empirical and population risks. We also examine the utilization of the exponential mechanism to attain favorable excess risk bounds for both a polynomial time sampling approach and an exponential time sampling approach. Despite the progress made, several interesting questions remain. 
There is still a gap between the upper and lower bounds for finding stationary points.
As noted in \cite{ABG+22}, it is quite challenging to beat the current $(\frac{\sqrt{d}}{n})^{2/3}$ empirical upper bound, and overcoming this challenge may require the development of new techniques. 
A potential avenue for improving the population rate for SOSP could be combining our drift-controlled framework with the tree-based private SpiderBoost algorithm in \cite{ABG+22}.
Additionally, it is worth exploring if it is possible to achieve better excess risk bounds within polynomial time, and what the optimal risk bound could be.

\section{Acknowledgement}
DG would like to thank Ruoqi Shen and Kevin Tian for several discussions.

\newpage
\bibliographystyle{alpha}	
\bibliography{ref.bib}
\newpage 
\appendix

\section{Omitted Proof of Section~\ref{sec:stationary_points}}
\subsection{Proof of Lemma~\ref{lm:good_gradient_estimator}}
\GoodGradientEstimator*
\begin{proof}
If $\drift_{\tau_t}=0$ happens, we use the first kind oracle to query the gradient, and hence $\nabla_{\tau_t}-\nabla F(x_{\tau_t})$ is zero-mean and $\nSG(2\zeta_1)$.
If $t=\tau_t$, Equation~\eqref{eq:graident_error} holds by the property of norm-subGaussian.

For each $\tau_t+1\le i\le t$, conditional on $\nabla_{i-1}$, we know $\Delta_i-(\nabla F(x_{i})-F(x_{i-1}))$ is zero-mean and $\nSG(\zeta_2\|x_i-x_{i-1}\|)$.
Note that
\begin{align*}
    \nabla_t-\nabla F(x_t)=\nabla_{\tau_t}-\nabla F(x_{\tau_t})+\sum_{i=\tau_t+1}^t[\Delta_i-(\nabla F(x_i)-\nabla F(x_{i-1}))].
\end{align*}
Equation~\eqref{eq:graident_error} follows from Lemma~\ref{lem:concentration_nSG}.

We know $\drift_{t-1}=\sum_{i=\tau_t+1}^{t}\|x_{i}-x_{i-1}\|^2\le \kappa$ almost surely by the design of the algorithm. 
By union bound, we know with probability at least $1-\omega$, for each $t\in[T]$, 
\begin{align*}
    \|\nabla_t-\nabla F(x_t)\|^2\le C(\zeta_2^2 \kappa+4\zeta_1^2)\cdot\log(Td/\omega)= \gamma^2/16.
\end{align*}
% We know if $\drift\le \gamma$, then $\|\nabla_t-\nabla \FD(x_t)\|\lesssim (\sigma_2^2 d\cdot\drift+\sigma_1^2 d)\log(T/\tau) $.

% See https://www.stat.cmu.edu/~arinaldo/Teaching/36755/F16/Scribed_Lectures/LEC0914.pdf
\end{proof}

\subsection{Discussion of Lemma~\ref{lem:escape_saddle_point}}
\label{subsec:proof_escape}
\EscapeSaddlePoint*

We briefly recap the proof of Lemma~\ref{lem:escape_saddle_point} in \cite{WCX19}.
One observation between the decreased function value, and the distance solutions moved is stated below:
\begin{lemma}[Lemma 11, \cite{WCX19}]
\label{lem:distance_function_value}
For each $t\in[\Gamma]$, we know
\begin{align*}
    \|x_{t+1}-x_0\|_2^2\le 8\eta (\Gamma (F(x_0)-F(x_\Gamma))+50\eta^2\Gamma\sum_{i\in[\Gamma]}\|\nabla_i-\nabla F(x_t)\|_2^2.
\end{align*}
\end{lemma}

The difference between our algorithm and the DP-GD in \cite{WCX19} is the noise on the gradient. 
Note that with high probability,
$\sum_{i\in[\Gamma]}\|\nabla_i-\nabla F(x_t)\|_2^2$ in our algorithm is controlled and small, and hence does not change the other proofs in \cite{WCX19}.
Hence if $F(x_0)-F(x_\Gamma)$ is small, i.e., the function value does not decrease significantly, we know $x_{t}$ is close to $x_0$.

Let $B_x(r)$ be the unit ball of radius $r$ around point $x$.
Denote the $(x)_\Gamma$ the point $x_\Gamma$ after running SGD mentioned in Lemma~\ref{lem:escape_saddle_point} for $\Gamma$ steps beginning at $x$.
With this observation, denote $B^{\gamma}(x_0):=\{x\mid x\in B_{x_0}(\eta\alpha), \Pr[F((x)_\Gamma)-F(x)\ge -\Phi]\ge \omega \}$.
\cite{WCX19} demonstrates the following lemma:
\begin{lemma}
\label{lem:bound_length_Bxr}
If $\|\nabla F(x_0)\|\le \alpha$ and $\smin(\nabla^2 F(x_0))\le -\sqrt{\rho\gamma}$, then the width of $B^{\gamma}(x_0)$ along the  along the minimum eigenvector of $\nabla^2F(x_0)$ is at most $
\frac{\omega\eta\gamma}{\log(1/\omega)}\sqrt{\frac{2\pi}{d}}$.
\end{lemma}

The intuition is that if two different points $x^1,x^2\in B_{x_0}(\eta\alpha)$, and $x^1-x^2$ is large along the minimum eigenvector, then with high probability, the distance between $\|(x^1)_\Gamma-(x^2)_\Gamma\|$ will be large, and either $\|(x^1)_\Gamma-x^1\|$ or $\|(x^2)_\Gamma-x^2\|$ is large, and hence either $F(x^1)-F((x^1)_{\Gamma})$ or $F(x^2)-F((x^2)_{\Gamma})$ is large.
The Lemma~\ref{lem:escape_saddle_point} follows from Lemma~\ref{lem:bound_length_Bxr} by using the Gaussian $\zeta_2$ to kick off the point.

\subsection{Proof of Lemma~\ref{lm:value_decrease}}
\ValueDecrease*
\begin{proof}
By the assumption on smoothness, we know
\begin{align*}
     F (x_{t+1})\le &  F (x_t)+\langle \nabla  F (x_t),x_{t+1}-x_t\rangle+\frac{M}{2}\|x_{t+1}-x_t\|^2\\
    =&  F (x_t)-\eta/2 \|\nabla_t\|^2-\langle \nabla F (x_t)-\nabla_t,\eta \nabla_t\rangle\\
    \le &  F (x_t)+\eta\|\nabla F (x_t)-\nabla_t\|\cdot \|\nabla_t\|-\frac{\eta}{2}\|\nabla_t\|^2.
\end{align*}

By Lemma~\ref{lm:good_gradient_estimator}, with probability at least $1-\omega$, for each $t\in [T]$ we have $\|\nabla F (x_t)-\nabla_t\|_2\le \gamma/4$.
Hence we know if $\nabla  F (x_t)\ge \gamma$, we have
\begin{align*}
    F(x_{t+1})-F(x_t)\le -\eta \|\nabla_t\|^2/6\le -\eta \gamma^2/6.
\end{align*}
\end{proof}

\subsection{Proof of Lemma~\ref{lem:guarantee_private_spider}}
\GPS*
\begin{proof}
By Lemma~\ref{lm:value_decrease}, we know if the gradient $\|\nabla F(x_t)\|\ge\gamma$, then with high probability that $F(x_{t+1})-F(x_t)\le -\eta\gamma^2/6$.
By Lemma~\ref{lem:escape_saddle_point}, if $x_t$ is a saddle point (with small gradient norm but the Hessian has a small eigenvalue), then with high probability that $F(x_{\Gamma+t})-F(x_{t})\le -\Omega(\frac{\gamma^{3/2}}{\sqrt{\rho}\log^3(\frac{dMB}{\rho \gamma\omega})})$, and the function values decrease $\Omega\big(\frac{\gamma^2}{M\log^4(\frac{dMB}{\rho \gamma\omega})}\big)$ on average for each step.

Recall the assumption that the risk is upper bounded by $B$, by our setting  $T=\Omega\big(\frac{BM}{\gamma^2}\log^4(\frac{dMB}{\rho \gamma\omega})\big)$, the statement is proved.
\end{proof}

\subsection{Proof of Lemma~\ref{lm:bound_total_times_large_drift}}
\BoundTotalTimesLargeDrift*

\begin{proof}
By Lemma~\ref{lm:value_decrease}, if $\| F (x_t)\|\ge \gamma$, we know $ F (x_{t+1})- F (x_t)\le -\eta\|\nabla_t\|^2/6$, and $ F (x_{t+1})- F (x_t)\le \eta\gamma^2$ otherwise.
Index the items in $K=\{t_1,t_2,\cdots,t_{|K|}\}$ such that $t_i<t_{i+1}$.
We know \begin{align*}
     F (x_{t_{i+1}})- F (x_{t_{i}})\le -\frac{1}{6\eta}\drift_{t_{i+1}}+(t_{i+1}-t_i)\gamma^2\eta\le -\frac{1}{6\eta}\kappa  +(t_{i+1}-t_i)\gamma^2\eta.
\end{align*}

Recall by the assumption that $\max_{y} F (y)-\min_{x} F (x)\le B$.
And hence $-B\le F (x_{t_{|L|}})- F (x_{t_1})\le -\frac{|K|}{6\eta}\kappa+T \gamma^2\eta$, and we know
\begin{align*}
    |K|\le O\big(\frac{B\eta}{\kappa}+T\gamma^2\eta^2/\kappa)=O(B\eta\log^4(\frac{dMB}{\rho\gamma\omega})/\kappa\big).
\end{align*}
\end{proof}

\subsection{Proof of Theorem~\ref{thm:sosp_emp}}
\SospEmp*
\begin{proof}
The privacy guarantee can be proved by composition theorems (Theorem~\ref{thm:bas_comp} and Theorem~\ref{thm:adv_comp}) and Lemma~\ref{lm:bound_total_times_large_drift}.

As for the utility, we know the $\oracle_1$ and $\oracle_2$ constructed in Equation~\eqref{eq:oracle_empirical} are first kind of $\sigma_1\sqrt{d}$ and second kind of $\sigma_2\sqrt{d}$ norm-subGaussian gradient oracle by Fact~\ref{fact:Gaussian}.
Hence by Lemma~\ref{lem:guarantee_private_spider}, the utility $\alpha_1$ satisfies that
\begin{align*}
    \alpha_1=& O(\sigma_1\sqrt{d}+\sigma_2\sqrt{d\kappa})\cdot\log^3(BMd/\rho\omega)\\
    =&O\Big(\frac{L\sqrt{dB\eta\log^2(1/\delta)/\kappa}}{n\epsilon}+\frac{M\log^3(ndMB/\omega)\sqrt{\log^2(1/\delta)BM}}{n\epsilon\alpha_1}\sqrt{d\kappa}\Big)\cdot\log^5(nBMd/\rho\omega).
\end{align*}
Choosing the best $\kappa$ demonstrates the bound on $\alpha_1$.
The bound for $\alpha_2$ follows from the value of $\alpha_1$ and Lemma~\ref{lem:gurantee_of_abovethreshold}.
Combining the two items in Lemma~\ref{lem:gurantee_of_abovethreshold}, we know with probability at least $1-\omega$, the output point $x$ of Algorithm~\ref{alg:privately_select} satisfies that
\begin{align*}
    \|\nabla \FD(x)\|\le \alpha_1+\frac{32\log(2T/\omega)G}{n\epsilon}, \text{ and } \smin(\nabla^2\FD(x))\ge -\sqrt{\rho\alpha_1}-\frac{32\log(2T/\omega)M}{n\epsilon}.
\end{align*}
Hence we know $x$ is an $\alpha_2$-SOSP for $\alpha_2$ stated in the statement.
\end{proof}

\subsection{Proof of Lemma~\ref{lem:emp_pop_sosp}}
\EmpPopSosp*
\begin{proof}
As for any $s\in S$, $\nabla f(x;s)-\nabla \FP(x)$ is zero-mean $\nSG(G)$.
Then the Hoeffding inequality for norm-subGuassians (Lemma~\ref{lem:concentration_nSG}) demonstrates with probability at least $1-\omega/2$, we have 
$\|\nabla F_{S}(x)-\nabla \FP(x)\|_2\le O\big(\frac{G\log(d/\omega)}{\sqrt{m}}\big)$.

As for the other term, we know for any $s\in S, \E[\nabla^2 f(x;s)-\nabla^2 \FP(x)]=0$, and $\|\nabla^2 f(x;s)-\nabla^2 \FP(x)\|_{op}\le 2M$ almost surely.
Hence applying Matrix Bernstein inequality (Theorem~\ref{thm:matrix_bern}) with $\sigma^2=4M^2m, t=O\big(\sqrt{m}M\log(d/\omega)\big)$, we know with probability at least $1-\omega/2$, $\|\nabla^2F_S(x)-\nabla^2\FP(x)\|_{op}\le t/m$.

Applying the Union bound completes the proof.
\end{proof}

\subsection{Proof of Theorem~\ref{thm:sosp_pop}}
\SospPop*
\begin{proof}
We should have all samples to be fresh to avoid dependency, and hence we need 
\begin{align*}
    b_1 \cdot|K| + b_2\cdot T\le n/2,
\end{align*}
which is satisfied by the parameter settings and Lemma~\ref{lm:bound_total_times_large_drift}.
As we never reuse a sample, the privacy guarantee follows directly from the Gaussian Mechanism~\cite{DR14}.
By lemma~\ref{lem:guarantee_private_spider}, we have
\begin{align*}
    &\frac{\alpha_1}{\log^3(nBMd/\rho\omega)}\\
    =& O(\sigma_1\sqrt{d}+\frac{G\sqrt{\log d}}{\sqrt{b_1}}+\sigma_2\sqrt{d\kappa}+\frac{M\sqrt{\kappa\log d}}{\sqrt{b_2}})\cdot\\
    =&O(\frac{GB\eta\sqrt{d\log(1/\delta)}}{n\epsilon\kappa}+\frac{BM^2\sqrt{\log(1/\delta)}}{n\epsilon\alpha_1^2}\sqrt{d\kappa}+\frac{G\sqrt{B\eta\log d}}{\sqrt{n\kappa}}+M\sqrt{\kappa}\frac{\sqrt{BM\log d}}{\sqrt{n}\alpha_1}).
\end{align*}

Setting $\kappa=\max (\frac{G^{4/3}B^{1/3}\log^{1/3}d}{M^{5/3}}(n)^{-1/3},(\frac{GB^{2/3}}{M^{5/3}})^{6/7} (\frac{\sqrt{d\log(1/\delta)}}{n\epsilon})^{4/7})$, we get
\begin{align*}
    \alpha_1=O\Big(\big((BGM\log d)^{1/3}\frac{1}{ n^{1/3}}+(G^{1/7}B^{3/7}M^{3/7}) (\frac{\sqrt{d\log(1/\delta)}}{n\epsilon})^{3/7}\big)\log^3(nBMd/\rho\omega)\Big).
\end{align*}

Then we use the other half fresh samples $\calD_2$ to find the point in the set by Algorithm~\ref{alg:privately_select}.
By Lemma~\ref{lem:gurantee_of_abovethreshold} and Lemma~\ref{lem:emp_pop_sosp}, we know with probability at least $1-\omega$, for some large enough constant $C>0$, the output point $x$ of Algorithm~\ref{alg:privately_select} satisfies that
\begin{align*}
    \|\nabla \FP(x)\|_2\le& \alpha_1+G(\frac{32\log(2T/\omega)}{n\epsilon}+\frac{C\log(dn/\omega)}{\sqrt{n}}),\\
  \smin(\nabla^2\FP(x))\ge & -\sqrt{\rho\alpha_1}-M(\frac{32\log(2T/\omega)}{n\epsilon}+\frac{C\log(dn/\omega)}{\sqrt{n}})
\end{align*}
Hence we know $x$ is an $\alpha_2$-SOSP for $\alpha_2$ stated in the statement.
The privacy guarantee follows from Basic composition and Lemma~\ref{lem:gurantee_of_abovethreshold}.
\end{proof}
\section{Omitted proof of Section~\ref{sec:risk_bound}}

\subsection{Proof of Lemma~\ref{lem:gen_error}}
\GenError*
\begin{proof}
We know how to bound the KL divergence by LSI:
\begin{align*}
    KL(\pi_{\calD},\pi_{\calD'}):=& \int \log\frac{\d \pi_{\calD}}{\d\pi_{\calD'}} \d \pi_{\calD}\\
    \le & \frac{\CLSI}{2} \E_{\pi_{\calD}}\left\|\nabla \log\frac{\d \pi_{\calD}}{\d \pi_{\calD'}} \right\|^2_2 \\
    \le& 2\CLSI G^2\beta^2/n^2.
\end{align*}
LSI can lead to Talagrand transportation inequality [Theorem 1 in \cite{OV00}], i.e., 
\begin{align*}
    W_2(\pi_{\calD},\pi_{\calD'})\lesssim \sqrt{\CLSI\cdot KL(\pi_{\calD},\pi_{\calD'})}=\CLSI G\beta/n.
\end{align*}

The generalization error is bounded by $O(\CLSI G^2\beta /n)$.
% In the convex case, $\CLSI=1/(\beta \mu)$, and hence the generalization error is $O(G^2/\mu n)$ which is independent of $\beta$.
Using Holley-Stroock perturbation, we know $\CLSI(\pi_\calD)\le \frac{\exp(\beta GD)}{\beta\mu}$ and hence the $W_2$ distance between $\pi_{\calD}$ and $\pi_{\calD'}$ can be bounded by $O(\frac{G\exp(\beta GD)}{n\mu})$.
The statement follows the Lipschitzness constant and Lemma~\ref{lem:stab_gen}.
\end{proof}

\subsection{Proof of Theorem~\ref{thm:zeroth_polytime_empirical}
}
\ZerothPolytimeEmpirical*
\begin{proof}
Denote $\pi(x)\propto \exp(-\beta(\FD(x)+\frac{\mu}{2}\|x\|_2^2))$.
By Lemma~\ref{lm:stroock_pert}, we know $\CLSI(\pi)\le \frac{1}{\beta\mu}\cdot\exp(\beta GD)$.
Plugging in the parameters and applying Theorem~\ref{thm:LSI_to_DP}, we get
\begin{align*}
    \frac{2G\beta}{n}\cdot \sqrt{\frac{\exp(\beta GD)}{\beta \mu}}\sqrt{3\log(1/\delta)}=O(1)\frac{GD\beta}{n\sqrt{d}}\sqrt{\exp(\beta GD)\log(1/\delta)}\le 1
\end{align*}
and hence prove the privacy  guarantee.

As for the empirical risk bound, by Lemma~\ref{lem:utility_nonconvex_sampling}, we know
\begin{align*}
    \E_{x\sim\pi}(\FD(x)+\frac{\mu}{2}\|x\|_2^2)-\min_{x^*\in\calK}(\FD(x^*)+\frac{\mu}{2}\|x^*\|_2^2)\lesssim \frac{d\log(\beta GD/d)}{\beta},
\end{align*}
and we know 
\begin{align*}
    \E_{x\sim\pi}\FD(x)-\min_{x^*\in\calK}\FD(x^*)\lesssim\frac{d\log(\beta GD/d)}{\beta}+\mu D^2.
\end{align*}
Replacing the value of $\beta$ achieves the empirical risk bound.

As for the population risk, we have
\begin{align*}
    &\E_{x\sim\pi}\FP(x)-\min_{y^*\in\calK}\FP(y^*)\\
    =& \E_{x\sim\pi}[\FP(x)-\FD(x)]+\E[\FD(x)-\min_{x^*\in\calK}\FD(x^*)]+\E[\min_{x^*\in\calD}\FD(x^*)-\min_{y^*\in\calK}\FP(y^*)]\\
    \le & \E_{x\sim\pi}[\FP(x)-\FD(x)]+\E[\FD(x)-\min_{x^*\in\calK}\FD(x^*)].
\end{align*}
We can bound $\E_{x\sim\pi}[\FP(x)-\FD(x)]\le O(\frac{G^2\exp(\beta GD)}{n\mu})\le O(\frac{GD\epsilon\log(n)}{n^{1-c}d\sqrt{\log(1/\delta)}})$ by Lemma~\ref{lem:gen_error} for an arbitrarily small constant $c>0$. 
Hence the empirical risk is dominated term compared to $\E_{x\sim\pi}[\FP(x)-\FD(x)]$, and we complete the proof.
\end{proof}

\subsection{Implementation}
\label{subsec:implementation}
We rewrite them below:
Let $\HF(x):=F(x)+r(x)$ where $r(x)$ is some regularizer, and $F=\E_{i\in I}f_i$ is the expectation of a family of $G$-Lipschitz functions.
\begin{algorithm}
\begin{algorithmic}[1]
\caption{AlternateSample, \cite{LST21}}
\label{alg:AlternateSample}
\STATE {\bf Input:} Function $\HF$, initial point $x_0\sim\pi_0$, step size $\eta$
\FOR{$t\in[T]$}
\STATE $y_t\leftarrow x_{t-1}+\sqrt{\eta}\zeta$ where $\zeta\sim\calN(0,I_d)$
\STATE Sample $x_t\leftarrow \exp(-\HF(x)-\frac{1}{2\eta}\|x-y_t\|^2_2)$
\ENDFOR
\STATE {\bf Output:} $x_T$
\end{algorithmic}
\end{algorithm}

\begin{theorem}[Guarantee of Algorithm~\ref{alg:AlternateSample}, \cite{CCSW22}]
\label{thm:alter_sample}
Let $\calK\subset\R^d$ be a convex set of diameter $D$, and $\HF:\calK\to\R$, and $\pi\propto\exp(-\HF)$ satisfies LSI with constant $\CLSI$.
Then set $\eta\ge 0$, we have
\begin{align*}
    R_q(\pi_t,\pi)\le \frac{R_q(\pi_0,\pi)}{(1+\eta/\CLSI)^{2t/q}},
\end{align*}
where $R_q(\pi',\pi)$ is
the $q$-th order of Renyi divergence between $\pi'$ and $\pi$.
% the total variance distance $\delta$, and $T=$, then the distribution of the output point $x_T$ has $\delta$ total variation distance to the distribution proportional to $\exp(-\HF)$.
\end{theorem}

To get a sample from $\exp(-\HF(x)-\frac{1}{2\eta}\|x-y_t\|_2^2)$, we use the rejection sampler from \cite{GLL22}, whose guarantee is stated below:
\begin{lemma}[Rejection Sampler, \cite{GLL22}]
\label{lem:rej_sampler}
If the step size $\eta\lesssim G^{-2}\log^{-1}(1/\delta_{inner})$ and the inner accuracy $\delta_{inner}\in(0,1/2)$, there is an algorithm that can return a random point $x$ that has $\delta_{inner}$ total variation distance to the distribution proportional to $\exp(-\HF(x)-\frac{1}{2\eta}\|x-y\|_2^2)$.
Moreover, the algorithm accesses $O(1)$ different $f_i$ function values and $O(1)$ samples from the density proportional to $\exp(-r(x)-\frac{1}{2\eta}\|x-y\|_2^2)$.
\end{lemma}

Combining Theorem~\ref{thm:zeroth_polytime_empirical}, Theorem~\ref{thm:alter_sample} and Lemma~\ref{lem:rej_sampler}, we can get the following implementation of the exponential mechanism for non-smooth functions.

\Implementation*
\begin{proof}
By Theorem~\ref{thm:zeroth_polytime_empirical}, it suffices to get a good sample from $\pi$ with density proportional to $\exp(-\beta(\FD(x)+\frac{\mu}{2}\|x\|_2^2))$ where $\beta=O(\frac{\epsilon\log(nd)}{GD\sqrt{\log(1/\delta))}}),\mu=\frac{d}{D^2\beta}$.
Set $q=1$, which gives that $R_q(\cdot, \cdot)$ is the KL-divergence.
Suppose we let $x_0$ is drawn from density proportional to $\exp(-\frac{\beta}{2}\mu\|x\|_2^2)$, then the KL divergence between $\pi_0$ and $\pi$ is bounded by $\exp(q\beta GD)$.

Now let $\pi_T^{(i)}$ be the distribution we get over $x_T$ from Algorithm~\ref{alg:AlternateSample} if we use an exact sampler for $i$ iterations, then the sampler of Lemma~\ref{lem:rej_sampler} for the remaining $T-i$ iterations. The output of Algorithm~\ref{alg:AlternateSample} that we actually get is $\pi_T^{(0)}$. Note that $\CLSI\le D^2n$, and $\eta\lesssim \beta^{-2}G^{-2}\log^{-1}(2T/\delta)$.
Setting 
$$T= O\left(\frac{\CLSI}{\eta}\log(\exp(q\beta GD)/\delta^2)\right) = \tilde{O}\left(\frac{n\epsilon^3 \log^3(d)\sqrt{\log(1/\delta)}}{GD}\right) 
% \log q n\epsilon^2\log^2n\log(T/\delta)\log(\exp(q\beta GD)/\epsilon)=O(n\epsilon\log^3n\log^{2}(n/\epsilon\delta)),
$$
we get $\delta_{inner} = \delta/2T$ in Lemma~\ref{lem:rej_sampler} and that $R_1(\pi_T^{(T)}, \pi) \leq \delta^2/8$. This implies the total variation distance between $\pi_T^{(T)}$ and $\pi$ is at most $\delta/2$ by Pinsker's inequality. Furthermore, by the post-processing inequality, the total variation distance between $\pi_T^{(i)}$ and $\pi_T^{(i+1)}$ is at most $\delta/2T$ for all $i$. Then by triangle inequality the total variation distance between $\pi_T^{(0)}$ and $\pi$ is at most $\delta$.
\end{proof}

\subsection{Proof of Theorem~\ref{thm:risk_exp}}
\label{subsec:proof_of_thm_risk_exp}
\RiskExp*
\begin{proof}
We pick a maximal packing $P$ of $O((D/r)^d)$ points, such that every point in $\calK$ is distance at most $r$ from some point in $P$. By $G$-Lipschitzness, the risk of any point in $P$ for the DP-ERM/SCO problems over $\calK$ are at most $Gr$ plus the risk of the same point for DP-ERM/SCO over $P$. The exponential mechanism over $P$ gives a DP-ERM risk bound of $O\left(\frac{GD}{\epsilon n} \log |P|\right)$. Next, note that the empirical loss of each point in $P$ is the average of $n$ random variables in $[0, GD]$ wlog. So, the expected maximum difference between the empirical and population loss of any point in $P$ is $O\left(\frac{GD \sqrt{\log |P|}}{\sqrt{n}}\right)$. Putting it all together we get a DP-SCO expected risk bound of:

\[O\left(Gr + GD \left(\frac{d \log(D/r)}{\epsilon n} + \frac{\sqrt{d \log(D/r)}}{\sqrt{n}}\right)\right).\]

This is approximately minimized by setting $r = Dd / \epsilon n$. This gives a bound of:

\[O\left(GD \left(\frac{d \log(\epsilon n / d)}{\epsilon n} + \frac{\sqrt{d \log(\epsilon n / d)}}{\sqrt{n}}\right)\right).\]
\end{proof}
\section{Extended related work}
\label{app:related}
\paragraph{First order stationary points.}
Progress towards privately finding a first-order stationary point is measured in %In the first order, previous results want to get a private solution $x\in\R^d$ to minimize 
($i$) the norm of the empirical gradient at the solution $x$, i.e., $\| \nabla \FD(x)\|$, and ($ii$) the norm of the population gradient, i.e., $\|\nabla \FP(x)\|$.
We summarize compare these first-order guarantees achieved by Algorithm~\ref{alg:DP-Spider-META} with previous algorithms in Table~\ref{tab:first_order}:

\begin{table}[h]
    \centering
    \begin{tabular}{|c|c|c|}
    \hline
    References & Empirical & Population
    \\
    \hline
    \cite{WYX17} & $\frac{d^{1/4}}{\sqrt{n}} $ & N/A \\
    \hline
      \cite{WX19}  &  $\frac{d^{1/4}}{\sqrt{n}}$ & $\frac{\sqrt{d}}{\sqrt{n}}$ \\
    \hline
    \cite{WJEG19} & $(\frac{\sqrt{d}}{n})^{2/3}$ &  N/A\\
    \hline
    \cite{ZCH+20}  & $\frac{d^{1/4}}{\sqrt{n}}$ &$\frac{d^{1/4}}{\sqrt{n}}$  \\
    \hline
    \cite{TC22} & $ \frac{1}{\sqrt{n}}+\Big(\frac{\sqrt{d}}{n}\Big)^{2/3}$ & N/A \\
    \hline
     \cite{ABG+22} & $\Big(\frac{\sqrt{d}}{n}\Big)^{2/3}$ & $\frac{1}{n^{1/3}}+(\frac{\sqrt{d}}{n})^{1/2}$ \\
     \hline
     %Ours &  $(\frac{\sqrt{d}}{n})^{2/3}$ & $\frac{1}{n^{1/3}}+(\frac{\sqrt{d}}{n})^{3/7}$ \\
     %\hline
    \end{tabular}
    \caption{Previous work in finding  first-order stationary points. %Algorithm~\ref{alg:DP-Spider-META} matches the state-of-the-art rate on empirical gradient norm, while is slightly worse than the SOTA on population gradient norm. 
    %Summary of previous results in the first order. 
    We omit logarithmic terms and dependencies on other parameters such as Lipschitz constant.
    ``N/A'' means we do not find an explicit result in the work.
    }
    \label{tab:first_order}
\end{table}
\vspace{-1em}
\paragraph{Second order stationary points.}
We say a point $x$ is a Second-Order Stationary Point (SOSP), or a local minimum of a twice differentiable function $g$ if $\|\nabla g(x)\|_2=0$ and $\smin(\nabla^2 g(x))\ge 0$.
Exact second-order stationary points can be extremely challenging to find \cite{GHJY15}. Instead, it is common to measure the progress in terms of how well the solution approximates an SOSP.
\begin{definition}[approximate-SOSP, \cite{AZB+17}]
We say $x\in\R^d$ is an $\alpha$-second order stationary point ($\alpha$-SOSP) for $\rho$-Hessian Lipschitz function $g$, if
\begin{align*}
    \|\nabla g(x)\|_2\le\alpha \;\bigwedge\; \smin(\nabla^2 g(x))\ge -\sqrt{\rho \alpha}.
\end{align*}
\end{definition}

\begin{table}[h]
    \centering
    \begin{tabular}{|c|c|c|}
    \hline
        References & Empirical & Population \\
    \hline
       \cite{WCX19}  & $\frac{d^{1/4}}{\sqrt{n}}$  & N/A \\
    \hline
    \cite{WX20} & $(\frac{d}{n})^{4/7} $ & N/A \\
    \hline
    \cite{GW23} & $(\frac{d}{n})^{1/2}$ & N/A \\
    \hline
    Ours &$(\frac{\sqrt{d}}{n})^{2/3}$ & $\frac{1}{n^{1/3}}+(\frac{\sqrt{d}}{n})^{3/7}$ \\
    \hline
    \end{tabular}
    \caption{Summary of previous results in finding $\alpha$-SOSP, where $\alpha$ is demonstrated in the Table. Omit the logarithmic terms and the dependencies on other parameters like Lipschitz constant.
    ``N/A'' means we do not find an explicit result in the work.}
    \label{tab:second-order}
\end{table}

Existing works in finding approximate SOSP privately give guarantees for the empirical function $\FD$. We improve upon the state-of-the-art result and give the first guarantee for the population function $\FP$, which is summarized in Table~\ref{tab:second-order}.

%\paragraph{Excess risk bound.}
%\Daogao{I copy it to the intro, and hence can comment it finally.}
%As for excess risk bound, the objective is to find a private solution $x^{priv}\in\R^d$ that minimizes the excess empirical risk:
%\begin{align*}
%    \E[\FD(x^{priv})]-\min_{x\in\calK}\FD(x),
%\end{align*}
%and the excess population risk:
%\begin{align*}
%    \E[\FP(x^{priv})]-\min_{x\in\calK}\FP(x).
%\end{align*}
%As the functions are non-convex, it is challenging to get a small empirical and population risk in polynomial time.
%With a mild smoothness assumption, \cite{WCX19} can get a private algorithm with polynomial running time and with upper $O(\frac{d\sqrt{\log(1/\delta)}}{\epsilon^2\log n})$ for both empirical and population risk.
%If we allow an exponential running time, \cite{GTU22} demonstrated $\Tilde{O}(\frac{d}{n\epsilon})$ upper bound for the empirical loss along with a nearly matching lower bound. 
%They leave it an open problem to obtain a tight bound for population risk of non-convex functions.
\end{document}